\documentclass{article}
\usepackage[arxiv]{kf-basic}

\usepackage[round]{natbib}

\usepackage[mode=tex]{standalone}
\usepackage{tikz}
\usetikzlibrary{calc}
\usepackage{pgfplots}
\pgfplotsset{compat=newest}
\usepgfplotslibrary{colormaps,patchplots}
\usepackage{graphicx}
\usetikzlibrary{arrows.meta,positioning,fit,calc,decorations.pathreplacing,backgrounds}
\usepackage{subcaption}

\title{Provable Target Sample Complexity Improvements as Pre‑Trained Models Scale}
\author[1,2]{Kazuto Fukuchi\thanks{fukuchi@cs.tsukuba.ac.jp, Corresponding author}}
\author[3,4]{Ryuichiro Hataya}
\author[4,5,6]{Kota Matsui}
\affil[1]{University of Tsukuba, Japan}
\affil[2]{RIKEN AIP, Japan}
\affil[3]{SB Intuitions Corp., Japan}
\affil[4]{Kyoto University, Japan}
\affil[5]{Shiga University, Japan}
\affil[6]{Institute of Science Tokyo, Japan}

\begin{document}

\maketitle

\begin{abstract}
  Pre-trained models have become indispensable for efficiently building models across a broad spectrum of downstream tasks. The advantages of pre-trained models have been highlighted by empirical studies on scaling laws, which demonstrate that larger pre-trained models can significantly reduce the sample complexity of downstream learning. However, existing theoretical investigations of pre-trained models lack the capability to explain this phenomenon. In this paper, we provide a theoretical investigation by introducing a novel framework, {\em caulking}, inspired by parameter-efficient fine-tuning (PEFT) methods such as adapter-based fine-tuning, low-rank adaptation, and partial fine-tuning. Our analysis establishes that improved pre-trained models provably decrease the sample complexity of downstream tasks, thereby offering theoretical justification for the empirically observed scaling laws relating pre-trained model size to downstream performance, a relationship not covered by existing results.
\end{abstract}

\section{Introduction}
The utilization of pre-trained models across diverse domains has become a prevalent strategy for developing models tailored to specific applications.
This approach enables the construction of highly accurate models even in scenarios where domain-specific data is limited.
For instance, in medical image recognition, numerous pre-trained models have been developed for a variety of tasks, such as disease diagnosis and the identification of diseased regions~\citep{wen2021rethinking}.
More recently, foundation models that handle different modalities, such as chest X-ray and brain CT images, within a unified framework have also been developed~\citep{azad2023foundational, wang2025self}.
Moreover, in fields such as drug discovery and materials science, pre-trained and foundation models capable of handling chemical structures are becoming powerful tools~\citep{xia2022systematic,pyzer2025foundation}.

The advantage of leveraging large pre-trained models has been underscored by empirical studies on {\em scaling laws}~\citep{henighanScalingLawsAutoregressive2020,mikamiScalingLawSyn2real2023}.
Scaling laws were first conceptualized by \citet{kaplanScalingLawsNeural2020} in the context of large language models~(LLMs), demonstrating that the performance of LLMs scales with model size, dataset size, and the amount of compute used for training.
The scaling laws of pre-trained models were further investigated by \citet{henighanScalingLawsAutoregressive2020} in the context of pre-trained autoregressive models, showing that larger pre-trained models can significantly reduce the sample complexity for fine-tuning downstream tasks.
\citet{mikamiScalingLawSyn2real2023} also demonstrated analogous scaling laws for pre-trained models in the context of synthetic-to-real transfer learning, showing that the performance of pre-trained models scales with the amount of pre-training data.
A special attention of this paper is {\em data scaling laws} of pre-trained models, where the performance for downstream tasks scales with the amount of pre-training data.

\begin{figure*}[t]
  \centering
  \begin{tikzpicture}[baseline={(0,0)}]
  \tikzset{
    model/.style={draw,thick,minimum width=7cm,minimum height=1.0cm,fill=blue!5},
    underlying/.style={draw,thick,minimum width=7cm,minimum height=1.0cm,fill=red!5},
    space/.style={draw,semithick,minimum height=1.0cm,fill=white},
    newadapter/.style={draw,semithick,minimum height=1.0cm,fill=green!20},
    repl/.style={-Latex,thick},
    divider/.style={draw,semithick}
  }


  \node[model, label={[align=center]above:Pre-trained model}] (modelA) {};
  \node[space, minimum width=.6*7cm] (spaceA) at (modelA.center) {};
  \node (adapterLabelA) at ($(modelA.west)!0.5!(modelA.east)$) {};
  \node[model, below=0.5cm of modelA.south] (model2A) {};
  \node[newadapter, minimum width=.6*7cm] (newA) at (model2A.center) {Adapter};
  \draw[repl] ($(modelA.south west)!0.5!(modelA.south east)$) to node[midway,right,xshift=4pt]{Caulking (replace)} (newA.north);
  \node[below=0.25cm of model2A.south, anchor=center, rotate=90] (equalA) {$\approx$};
  \node[underlying, below=0.25cm of equalA.center] (underlyingA) {underlying function};

  \node[model, right=2.5cm of modelA, label={[align=center]above:Pre-trained model}] (modelB) {};
  \node[space, minimum width=.4*7cm] (spaceB) at (modelB.center) {};
  \node (adapterLabelB) at ($(modelB.west)!0.5!(modelB.east)$) {};
  \node[model, below=0.5cm of modelB.south] (model2B) {};
  \node[newadapter, minimum width=.4*7cm] (newB) at (model2B.center) {Adapter};
  \draw[repl] ($(modelB.south west)!0.5!(modelB.south east)$) to node[midway,right,xshift=4pt]{Caulking (replace)} (newB.north);
  \node[below=0.25cm of model2B.south, anchor=center, rotate=90] (equalB) {$\approx$};
  \node[underlying, below=0.25cm of equalB.center] (underlyingB) {underlying function};

  \node[model, right=2.5cm of modelB, label={[align=center]above:Pre-trained model}] (modelC) {};

  \node[space, minimum width=.2*7cm] (spaceC) at (modelC.center) {};
  \node (adapterLabelC) at ($(modelC.west)!0.5!(modelC.east)$) {};
  \node[model, below=0.5cm of modelC.south] (model2C) {};
  \node[newadapter, minimum width=.2*7cm] (newC) at (model2C.center) {Adapter};
  \draw[repl] ($(modelC.south west)!0.5!(modelC.south east)$) to node[midway,right,xshift=4pt]{Caulking (replace)} (newC.north);
  \node[below=0.25cm of model2C.south, anchor=center, rotate=90] (equalC) {$\approx$};
  \node[underlying, below=0.25cm of equalC.center] (underlyingC) {underlying function};

  \draw[repl] ($(modelA.north west)+(0,0.8cm)$) -- node[above,align=center]{Pre-trained model scale} node[pos=0,above,anchor=south west,align=center]{Small $m$} node[pos=1,above,anchor=south east,align=center]{Large $m$} ($(modelC.north east)+(0,0.8cm)$);

  \begin{scope}[on background layer]
    \draw[repl] ([xshift=-0.5cm]modelA.west) -- ([xshift=+0.5cm]modelA.east);
    \draw[repl] ([xshift=-0.5cm]modelB.west) -- ([xshift=+0.5cm]modelB.east);
    \draw[repl] ([xshift=-0.5cm]modelC.west) -- ([xshift=+0.5cm]modelC.east);

    \draw[repl] ([xshift=-0.5cm]model2A.west) -- ([xshift=+0.5cm]model2A.east);
    \draw[repl] ([xshift=-0.5cm]model2B.west) -- ([xshift=+0.5cm]model2B.east);
    \draw[repl] ([xshift=-0.5cm]model2C.west) -- ([xshift=+0.5cm]model2C.east);

    \draw[repl] ([xshift=-0.5cm]underlyingA.west) -- ([xshift=+0.5cm]underlyingA.east);
    \draw[repl] ([xshift=-0.5cm]underlyingB.west) -- ([xshift=+0.5cm]underlyingB.east);
    \draw[repl] ([xshift=-0.5cm]underlyingC.west) -- ([xshift=+0.5cm]underlyingC.east);

  \end{scope}

\end{tikzpicture}
  \caption{A conceptual illustration of caulking. Blue boxes represent pre-trained models, and red boxes represent underlying functions. The horizontal axis represents the source sample size $m$, which corresponds to the scale of the pre-trained model.}
  \label{fig:caulking}
\end{figure*}
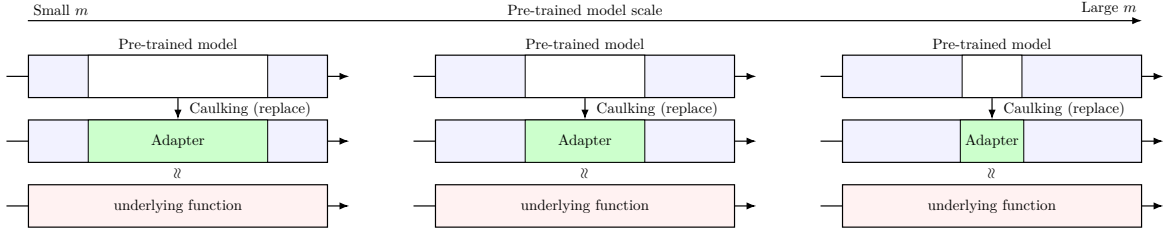

The effectiveness of pre-trained models is also supported by substantial theoretical research, including work on few-shot learning~\citep{duFewShotLearningLearning2020} and in-context learning~\citep{baiTransformersStatisticiansProvable2023,kimTransformersAreMinimax2024}.
Most of these studies attempt to demonstrate the advantage of pre-trained models by establishing an upper bound on the error of the learning algorithm, characterized by the source sample size $m$~(used to construct the pre-trained model) and the target sample size $n$~(sample size for the downstream task).
Specifically, the error rate obtained in these analyses is often expressed as:
\begin{align}
  \mathbb{E}[E_Q(f_n)] \le m^{-\alpha} + n^{-\beta}, \label{eq:rate-exist}
\end{align}
where $E_Q(f_n)$ denotes the error of the estimated regression function $f_n$ under the target distribution $Q$ of the downstream task, and $\alpha,\beta > 0$ are constants that depend on the complexity of the underlying function class and the learning algorithm.
If $\beta$ is larger than the rate achievable by training from scratch, \zcref{eq:rate-exist} implies that the pre-trained model is beneficial, provided that $m$ is sufficiently large such that $m^{-\alpha} \le n^{-\beta}$.

However, these theoretical results have limited capacity to explain the data scaling laws of pre-trained models.
Within the framework of the theoretical analyses in \zcref{eq:rate-exist}, the data scaling law for pre-trained models can be interpreted as the increase in source sample size $m$~(corresponding to the scale of the pre-trained model) leading to improved sample complexity with respect to the target sample size $n$~(the sample complexity of the downstream task).
Yet, the sample complexity for $n$ is characterized by $\beta$ in \zcref{eq:rate-exist}, which remains invariant to the scale of the pre-trained model, $m$.
Thus, the existing results showing \zcref{eq:rate-exist} do not fully capture the data scaling laws of pre-trained models.

To address this gap, this paper aims to explore the following critical research question:
\begin{quote}
  What properties of a pre-trained model can ensure improved sample complexity for the target sample size as the source sample size increases?
\end{quote}
Identifying such characteristics of pre-trained models can lead to a more nuanced understanding of their recent success, particularly the data scaling laws of pre-trained models, and may result in the development of more efficient methodologies for constructing pre-trained models.

\paragraph{Contributions.}
The primary contribution of this paper is to clarify the properties of pre-trained models that give rise to the data scaling laws observed in practice.
Specifically, we show that under a certain assumption on the pre-trained model, the error rate satisfies
\begin{align}
  \mathbb{E}[E_Q(f_n)] \le n^{-\beta + o(1)}, \label{eq:rate-our}
\end{align}
for some constant $\beta > 0$, where $o(1)$ is a finite term that decreases with $m$ and vanishes as $m \to \infty$.
The rate in \zcref{eq:rate-our} demonstrates that the sample complexity with respect to $n$ improves as the source sample size $m$ increases, since the exponent of $n$ decreases with $m$.
Thus, the bound in \zcref{eq:rate-our} provides a theoretical explanation for the data scaling laws of pre-trained models.

To achieve the rate in \zcref{eq:rate-our}, we introduce a novel concept called {\em caulkability}, which is conceptually illustrated in \zcref{fig:caulking}.
Caulkability requires that the pre-trained model~(blue boxes in \zcref{fig:caulking}) possesses a hierarchical structure, such that the underlying function~(red boxes in \zcref{fig:caulking}) can be approximated by inserting another function, termed the {\em adapter model}~(green boxes in \zcref{fig:caulking}), into this hierarchy.
We demonstrate that if the complexity of the adapter model decreases as the source sample size $m$ increases~(i.e., the depth of the adapter model shrinks along the horizontal axis representing $m$ in \zcref{fig:caulking}), it is possible to construct a learning algorithm, called {\em empirical caulking}, that achieves the error rate in \zcref{eq:rate-our}.
This result shows that reducing the complexity of the adapter model is a sufficient condition for attaining the rate in \zcref{eq:rate-our}, thereby answering the research question above.
This also highlights that parameter-efficient fine-tuning (PEFT) methods can be effective in learning with pre-trained models, as we can interpret leveraging low complexity adapter models as a form of PEFT.

We further provide empirical evidence that supports our theoretical results.
In particular, our experiments on fine-tuning CNNs and integrating vision capabilities into LLMs demonstrate that larger pre-trained models can be adapted to downstream tasks using shallower adapter models.
These findings reinforce our theoretical claim that adapting larger pre-trained models leads to a reduction in the sample complexity required for downstream tasks.

All missing proofs are deferred to \zcref{sec:missing-proofs}.

\section{Related Work}

\paragraph{Domain Adaptation Theory.} Domain adaptation is a major subfield of transfer learning that addresses the scenario where the training (source) data and test (target) data come from different distributions even though the learning task is the same~\citep{redko2020survey}.
Domain adaptation is, in essence, a transfer-learning approach that leverages abundant source-domain data together with a small amount of explicitly provided target-domain data (either labeled or unlabeled), and theoretical analyses are often conducted under this assumption.
A rich theory has been developed to understand the generalization performance under domain shift.
Seminal work by \citet{ben2006analysis,ben2010theory} introduced generalization bounds for domain adaptation that relate the target error to three components, namely source domain error, domain discrepancy, and an irreducible term (known as the joint error).
The general bound can be stated informally as: for any hypothesis $h$ in a given class, the target risk $R_T(h)$ is bounded by the source risk $R_S(h)$ plus a discrepancy between the domains and a constant term for labeling mismatch.
This result implies that if two domains are very similar (small discrepancy) and also share an underlying label function, then a model will generalize well across domains, whereas large distribution gaps or conflicting label definitions will lead to a larger generalization error.
Numerous theoretical results have refined these ideas.
Researchers have proposed various discrepancy metrics for different settings, e.g., the symmetric difference $\mathcal{H}$-divergence for classification and regression~\citep{ben2010theory}, and extended the theory to multi-source adaptation and other scenarios~\citep{mansour2008domain}.
Other work analyzes, e.g., maximum mean discrepancy~\citep{redko2019advances} or the Wasserstein distance~\citep{shen2018wasserstein}, with their own theoretical guarantees.

Apart from discrepancy-based theories, alternative approaches have also been explored.
For instance, \cite{kpotufe2017lipschitz, ma2023optimally, feng2023towards} investigate upper bounds on the error of methods that account for distribution shift between the source and target domains using density ratios.
In addition, \cite{kpotufe2021marginal,pathak2022new,galbraith2023classification,fujikawaHarnessingPowerVicinityInformed2025} attempt to analyze error bounds using a similarity measure based on hyperspheres in a metric space.
While these analyses appear to yield sound error bounds, several limitations can be pointed out.
For example, in density-ratio–based methods, the analysis is often carried out under the assumption that the true density ratio is known; in practice, it is unknown, so the estimation error is not taken into account.
Moreover, in both approaches, when the supports of the source and target data distributions are significantly mismatched, the similarity measure can diverge.
Therefore, existing domain adaptation theory is insufficient as a tool for analyzing the error of pre-trained models; new tools need to be developed.

\paragraph{Deep Learning Theory for Pre-trained Models.}
There has been substantial theoretical progress in understanding the statistical performance of deep learning models~\citep{schmidt-hieberNonparametricRegressionUsing2020,suzukiDeepLearningAdaptive2021,kohlerRateConvergenceFully2021a,suhApproximationNonparametricEstimation2022,nishikawaNonlinearTransformersCan2025,damianNeuralNetworksCan2022a,damianSmoothingLandscapeBoosts2023}.
This line of research has been further extended to analyze the statistical properties of pre-trained models, particularly in establishing the error rate in \zcref{eq:rate-exist}.
For instance, \citet{duFewShotLearningLearning2020} study few-shot learning problems within the linear feature class and linear outcome model, demonstrating the error rate of the form in \zcref{eq:rate-exist}.
\citet{baiTransformersStatisticiansProvable2023} investigate the statistical performance of transformers for in-context learning under the linear regression model, also obtaining the form in \zcref{eq:rate-exist}.
\citet{kimTransformersAreMinimax2024} analyze transformers for in-context learning in the Besov space, again showing the error rate of the form in \zcref{eq:rate-exist}.
As discussed in the introduction, these results share the form in \zcref{eq:rate-exist} and therefore do not account for the scaling laws observed in pre-trained models.

\paragraph{Large Language Models.}
The proposed concept of caulking is closely related to the recent advances in large language models~(LLMs).
In particular, recent training methods for multimodal language models~\citep{liBLIP2BootstrappingLanguageImage2023,baiQwenVLVersatileVisionLanguage2023,liuVisualInstructionTuning2023,daiInstructBLIPGeneralpurposeVisionLanguage2023,gosthipaty2024nanovlm,liu2025nvila} and parameter-efficient fine-tuning (PEFT) techniques~\citep{houlsbyParameterEfficientTransferLearning2019,guoParameterEfficientTransferLearning2021,huLoRALowRankAdaptation2021} can be viewed as analogous to caulking.
For instance, when constructing a vision-language model using a pre-trained visual feature extractor and an LLM, a typical strategy is to learn a mapping from visual features to textual features, referred to as an adapter model, using an image-text pair dataset.
The final vision-language model is then formed by inserting the learned adapter model between the visual feature extractor and the LLM, which is analogous to the caulking process.
Similarly, PEFT methods involve replacing or augmenting the pre-trained model with a component of lower complexity, which is also analogous to caulking.

\section{Learning with Pre-trained Model via Caulking}\label{sec:setup}

\paragraph{Notations.}
Let $\mathcal{E} \in \mathfrak{Z}$ be an event in a probability space $(\mathcal{Z}, \mathfrak{Z}, \nu)$.
The complement of $\mathcal{E}$ is denoted by $\mathcal{E}^c$, and its probability by $\mathbb{P}_\nu\Bab{\mathcal{E}}$.
For a random variable $X$ defined on $(\mathcal{Z}, \mathfrak{Z}, \nu)$ with values in a measurable space $(\mathcal{X}, \mathfrak{X})$, its expectation is denoted by $\mathbb{E}_\nu[X]$.
For $p \in [1, \infty)$, the $L^p$-norm of a function $f:\mathcal{Z}\to\mathbb{R}$ is denoted by $\Vab{f}_{L^p(\mathcal{Z})} = (\int_{\mathcal{Z}} |f|^p d\lambda)^{1/p}$, where $\mathcal{Z} \subseteq \mathbb{R}^d$ for some $d \in \mathbb{N}$ and $\lambda$ is the Lebesgue measure.
For a measure $\nu$ on a measurable space $(\mathcal{Z}, \mathfrak{Z})$, the $L^p(\nu)$-norm of $f:\mathcal{Z}\to\mathbb{R}$ is defined as $\Vab{f}_{L^p(\nu)} = (\int |f|^p d\nu)^{1/p}$ for $p \in [1, \infty)$ and $\Vab{f}_{L^\infty(\nu)} = \inf\Bab{b > 0 : \nu(\Bab{z \in \mathcal{Z} : |f(z)| \le b})=1}$.

\paragraph{Problem Setup.}
We consider a nonparametric regression problem in the presence of a pre-trained model.
Let $X \in \mathcal{X}$ and $Y \in \Omega \subset \mathbb{R}$ denote the covariates and outcome, respectively, and let $P$ and $Q$ denote the source and target distributions.
Under the target distribution $Q$, assume that $X$ and $Y$ follow the nonparametric regression model:
\begin{align}
  Y = f^*(X) + \xi,
\end{align}
where $f^*$ is the regression function of interest and $\xi$ is a zero-mean noise term independent of $X$.
Suppose the learner has access to a pre-trained model $f_{\mathrm{pre}}$ constructed from a source sample of size $m$ drawn from $P$.
The learner's objective is to estimate $f^*$ using a target sample of size $n$ from $Q$ in conjunction with the pre-trained model $f_{\mathrm{pre}}$.
Let $f_n$ denote the estimated regression function.
Then, the accuracy of $f_n$ is evaluated via the $L^2(Q_X)$ distance from $f^*$, defined as
\begin{align}
  E_Q(f) = \Vab{f - f^*}_{L^2(Q_X)}^2,
\end{align}
where $Q_X$ is the marginal distribution of $X$ under $Q$.

\paragraph{Motivating Example.}

\begin{figure}[t]
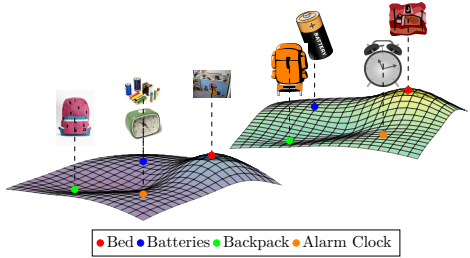

  \centering
  \includestandalone[width=0.5\textwidth]{manifolds}
  \caption{A motivating example illustrating the successful utilization of a pre-trained model.}
  \label{fig:manifolds}
\end{figure}

\zcref{fig:manifolds} illustrates a conceptual example that demonstrates the effective use of a pre-trained model, using images inspired by the Office-Home dataset~\citep{venkateswaraDeepHashingNetwork2017}.
The figure consists of two manifolds, each representing the probability that an image is classified as Bed~($f^*$), based on the feature space for two domains: Real-world~($P$) and Clipart~($Q$).
Points on the manifolds correspond to features, with images connected by dashed lines.
The manifolds differ only by rotations and translations.
A pre-trained model is constructed to capture the complex structure of the left manifold using real-world images (source sample).
Subsequently, an accurate estimate of $f^*$ for the clipart images (target sample) on the left manifold can be obtained by aligning the right manifold with the left one through appropriate transformations.
Learning the complex manifold structure from scratch requires a larger sample size than learning only the rotations and translations (i.e., linear transformations).
Therefore, leveraging the pre-trained model enables high predictive accuracy with a small target sample, highlighting the potential effectiveness of pre-trained model utilization.

\paragraph{Caulking.}
The scenario depicted in \zcref{fig:manifolds} can be described by a hierarchical structure, formalized by the concept of caulkability.
\begin{definition}[Caulkability]\label{def:caulkability}
  Let $\mathcal{G}$ be a class of functions $g:\mathcal{Z}_i \to \mathcal{Z}_o$, where $\mathcal{Z}_i$ and $\mathcal{Z}_o$ are arbitrary domains.
  Given $\epsilon > 0$ and $\mathcal{G}$, a regressor $f^*:\mathcal{X}\to\mathbb{R}$ is $(\epsilon, \mathcal{G})$-caulkable by $(g_h, g_e)$, where $g_e:\mathcal{X}\to\mathcal{Z}_i$ and $g_h:\mathcal{Z}_o\to\mathbb{R}$, with respect to the $L^2(Q_X)$-norm if
  \begin{align}
    \inf_{g_a \in \mathcal{G}} \Vab{f^* - g_h \circ g_a \circ g_e}_{L^2(Q_X)} \le \epsilon. \label{eq:caulkability}
  \end{align}
\end{definition}
To successfully leverage the pre-trained model, we assume that the ideal regressor $f^*$ is caulkable by the pre-trained model $f_{\mathrm{pre}} = (g_h, g_e)$.
In \zcref{def:caulkability}, $g_e$ denotes a mapping from the image to features, corresponding to the dashed lines in \zcref{fig:manifolds}; the manifold in \zcref{fig:manifolds} is parameterized by these features, where the height corresponds to the output of $g_h$ and the spatial position corresponds to its input; and $g_a$ denotes a domain-specific transformation, corresponding to the rotations and translations for the left and right manifolds in \zcref{fig:manifolds}, respectively.
\zcref{eq:caulkability} indicates that the ideal regressor $f^*$ can be (approximately) recovered by inserting a target-specific transformation $g_a$ into the pre-trained model $f_{\mathrm{pre}}$.
We refer to this property of the pre-trained model as caulkability, expressing the situation where the ideal regressor $f^*$ is recovered by filling the gap between $g_h$ and $g_e$.
We refer to $g_h$ and $g_e$ as the {\em head model} and {\em feature extractor model}, respectively, and to $g_a$ as the {\em adapter model}.

\paragraph{Empirical Caulking.}
We propose a learning framework that leverages the pre-trained model, termed {\em empirical caulking}.
Let $(X_1, Y_1), ..., (X_n, Y_n)$ be a target sample drawn from $Q$. Suppose that the ideal regressor $f^*$ is $(\epsilon_n, \mathcal{G})$-caulkable by the pre-trained model $f_{\mathrm{pre}} = (g_h, g_e)$, where $\epsilon_n$ is a constant possibly depending on the target sample size $n$, and $\mathcal{G}$ is a given class of functions.
We design a learning algorithm for $f_n$ by inserting the fine-tuned adapter model $g_a$ into the pre-trained model $f_{\mathrm{pre}}$.
Given a class $\mathcal{G}_n \subseteq \mathcal{G}$ of possible $g_a$, potentially depending on $n$, define $\mathcal{F}_n = \Bab{g_h \circ g_a \circ g_e : g_a \in \mathcal{G}_n}$.
$\mathcal{F}_n$ thus represents the set of functions obtainable by inserting an adapter model $g_a \in \mathcal{G}_n$ into the pre-trained model $f_{\mathrm{pre}}$.
The estimated regressor $f_n$ is then defined as
\begin{align}
  f_n = \argmin_{f \in \mathcal{F}_n} \frac{1}{n} \sum_{i=1}^n \ab(f(X_i) - Y_i)^2.
\end{align}
This learning algorithm is equivalent to the least squares estimator, but restricted to the specific function class $\mathcal{F}_n$.
\begin{remark}
  Our current formulation of empirical caulking covers the most parameter-efficient fine-tuning methods~(PEFTs), whereas it cannot directly cover the case of LoRA~\citep{huLoRALowRankAdaptation2021}.
  Nevertheless, LoRA can be interpreted as a method for inserting several layers into the pre-trained model.
  Specifically, letting $g_H\circ \cdots \circ g_1$ be an $H$-layer pre-trained model, consider the modified version $h_{H+1} \circ h_H\circ \cdots \circ h_1 \circ h_0$, where $h_i(x, y) = (g_i(x), y)$ for $i = 1, ..., H$, $h_0(x) = (x, x)$, and $h_{H+1}(x, y) = x$.
  Then, the composition of the $h_i$ is equivalent to the composition of the $g_i$. LoRA can be seen as inserting adapter layers $h_{1+1/2}, ..., h_{H+1/2}$ into this modified pre-trained model.
  Specifically, we can recover LoRA by setting $h_{i + 1/2}(x, y) = (x + \Delta W y, x + \Delta W y)$, for $i = 1, ..., H$, where $\Delta W$ is low rank.
  Hence, extending our result to cover the case of inserting multiple adapter layers would cover LoRA, which is a crucial future direction of this work.
\end{remark}

\section{Error Analysis of Empirical Caulking}\label{sec:error-analysis}
In this section, we provide a rigorous analysis of the error associated with empirical caulking, utilizing an appropriate complexity measure for the class $\mathcal{G}$.
Throughout this section, we assume that $\mathcal{Z}_o$ is equipped with a norm $\Vab{\cdot}$, and we define the $L^\infty$-norm of an adapter model $g_a:\mathcal{Z}_i\to\mathcal{Z}_o$ as $\Vab*{g_a}_{L^\infty} = \sup_{z \in \mathcal{Z}_i}\Vab*{g_a(z)}$.

\paragraph{Assumption.}
We make the following assumptions throughout our analysis:
\begin{assumption}\label{asm:regularity}
  The following conditions hold:
  \begin{enumerate}
    \item The noise $\xi$ is sub-Gaussian with variance proxy $\sigma^2 < \infty$, i.e., $\mathbb{E}[\exp(\lambda \xi)] \le \exp(\lambda^2 \sigma^2/2)$ for all $\lambda \in \mathbb{R}$.
    \item $\Omega$ is bounded, that is, there exists a constant $\Delta_\Omega > 0$ such that $\sup_{x,y \in \Omega} |y - x| \le \Delta_\Omega$.
  \end{enumerate}
\end{assumption}
\zcref{asm:regularity} is standard in the literature on nonparametric regression and statistical learning theory. Similar assumptions are made in prior works such as \citet{schmidt-hieberNonparametricRegressionUsing2020} and \citet{suzukiDeepLearningAdaptive2021}, which analyze the theoretical properties of deep learning and nonparametric regression methods.

\paragraph{Complexity of $\mathcal{G}$.}
To characterize the error of empirical caulking, we adopt a complexity measure for $\mathcal{G}$ based on both the approximation error and the covering number.
Given $\delta > 0$, a class $\mathcal{G}$, and a norm $\Vab{\cdot}$ defined on $\mathcal{G}$, the $\delta$-covering number $N(\delta, \mathcal{G}, \Vab{\cdot})$ is the minimal number of balls of radius $\delta$ (with respect to $\Vab{\cdot}$) needed to cover $\mathcal{G}$.
The following definition formalizes the notion of complexity.
\begin{definition}[$\beta$-complexity]\label{def:beta-complexity}
  Let $\beta > 0$. A class $\mathcal{G}$ of functions $g:\mathcal{Z}_i\to\mathcal{Z}_o$ is $\beta$-complex if, for each $J \in \mathbb{N}$, there exists a class $\mathcal{G}_J$ such that
  \begin{align}
    \ln\ab(N(\delta, \mathcal{G}_J, \Vab{\cdot}_{L^\infty})) \le c J\mathrm{polylog}(J, 1/\delta),
  \end{align}
  and
  \begin{align}
    \sup_{g^* \in \mathcal{G}}\inf_{g \in \mathcal{G}_J} \Vab{g - g^*}_{L^\infty} \le c' J^{-\beta},
  \end{align}
  for some constants $c,c' > 0$, where $\mathrm{polylog}$ denotes a polylogarithmic factor in its arguments.
\end{definition}
\zcref{def:beta-complexity} characterizes the complexity of $\mathcal{G}$ through the trade-off between approximation capability and the complexity of the classes $\mathcal{G}_J$.
This concept is also used in the analysis of nonparametric regression error bounds for classes such as sparse ReLU networks, as in \citep{schmidt-hieberNonparametricRegressionUsing2020,suzukiDeepLearningAdaptive2021,chenNonparametricRegressionLowdimensional2022}.
A concrete application of \zcref{def:beta-complexity} is provided in \zcref{sec:example}.

\paragraph{H\"older Continuity.}
Alongside $\beta$-complexity, we employ H\"older continuity to quantify the regularity of the pre-trained model $f_{\mathrm{pre}}$.
\begin{definition}[H\"older continuity]\label{def:holder-continuity}
  Let $g$ be a function from $\mathcal{X}$ to $\mathcal{Z}$, where $\mathcal{X}$ and $\mathcal{Z}$ are equipped with norms $\Vab{\cdot}_X$ and $\Vab{\cdot}_Z$, respectively. For $\alpha \in (0,1]$, $g$ is said to be $\alpha$-H\"older continuous if there exists a constant $C > 0$ such that
  \begin{align}
    \Vab{g(x) - g(y)}_{Z} \le C \Vab{x - y}_X^\alpha,
  \end{align}
  for all $x, y \in \mathcal{X}$.
\end{definition}

\paragraph{Error Bound.}
We now present an error bound for empirical caulking when $\mathcal{G}$ is $\beta$-complex and the ideal function $f^*$ is $(\epsilon_n, \mathcal{G})$-caulkable by a pre-trained model $f_{\mathrm{pre}}$.
\begin{theorem}\label{thm:error-bound}
  Let $\alpha \in (0,1]$ and $\beta > 0$.
  Suppose that $f^*$ is $(n^{-\frac{2\alpha\beta}{2\alpha\beta+1}}, \mathcal{G})$-caulkable by $f_{\mathrm{pre}}=(g_h, g_e)$ for some $\beta$-complex class $\mathcal{G}$, and $g_h$ is $\alpha$-H\"older continuous.
  Let $f_n$ denote the estimated regressor obtained by empirical caulking with $\mathcal{G}_n = \mathcal{G}_J$, where $\mathcal{G}_J$ is as in \zcref{def:beta-complexity} and $J=\lceil n^{1/(2\alpha\beta+1)} \rceil$.
  Under \zcref{asm:regularity}, we have
  \begin{align}
    \mathbb{E}_{Q^n}[E_Q(f_n)] \le C n^{-\frac{2\alpha\beta}{2\alpha\beta+1}}\mathrm{polylog}(n),
  \end{align}
  for some constant $C > 0$.
\end{theorem}
Building on \zcref{thm:error-bound}, we now examine the effectiveness of a pre-training algorithm.
Consider a pre-training algorithm that yields a pre-trained model $f_{\mathrm{pre},m}$ such that $f^*$ is caulkable with a sequence of adapter model classes $\mathcal{G}_m$ whose complexity decreases as $m$ increases.
Specifically, suppose that for each $m$, the pre-trained model $f_{\mathrm{pre},m} = (g_{h,m}, g_{e,m})$ enables $f^*$ to be $(\epsilon_n, \mathcal{G}_m)$-caulkable with a $(\beta-\gamma_m)$-complex class $\mathcal{G}_m$, where $\gamma_m$ is a bounded term that vanishes as $m$ grows.
For such a pre-training algorithm, we obtain the following corollary.
\begin{corollary}\label{cor:error-bound-for-m}
  Let $\alpha \in (0,1]$ and $\beta > 0$.
  Let $\gamma_m < \beta$ be a bounded sequence depending on $m$.
  Suppose that $f^*$ is $(n^{-\frac{2\alpha\beta}{2\alpha\beta+1} + \gamma_m}, \mathcal{G}_m)$-caulkable by $f_{\mathrm{pre},m} = (g_{h,m}, g_{e,m})$ for some $(\beta - \gamma_m)$-complex class $\mathcal{G}_m$, and $g_{h,m}$ is $\alpha$-H\"older continuous.
  Let $f_n$ denote the estimated regressor obtained by empirical caulking with $\mathcal{G}_{n} = \mathcal{G}_J$, where $\mathcal{G}_J$ is as in \zcref{def:beta-complexity} and $J=\lceil n^{1/(2\alpha(\beta-\gamma_m)+1)} \rceil$.
  Under \zcref{asm:regularity}, we have
  \begin{align}
    \mathbb{E}_{Q^n}[E_Q(f_n)] \le C n^{-\frac{2\alpha\beta}{2\alpha\beta+1}\ab(1 - \frac{\gamma_m}{\beta(\alpha(\beta- \gamma_m) + 1)})}\mathrm{polylog}(n),
  \end{align}
  for some constant $C > 0$.
\end{corollary}
\zcref{cor:error-bound-for-m} indicates that empirical caulking achieves a rate of $n^{-\frac{2\alpha\beta}{2\alpha\beta+1} + o(1)}$ up to a logarithmic factor, where $o(1)$ is a bounded term that vanishes as $m$ increases, since $\frac{\gamma_m}{\beta(\alpha(\beta- \gamma_m) + 1)}$ decreases as $\gamma_m$ decreases.
Therefore, \zcref{cor:error-bound-for-m} demonstrates that empirical caulking can improve target sample complexity as the source sample size increases, provided that $f^*$ is caulkable with a sequence of classes $\mathcal{G}_m$ of decreasing complexity.
This result offers partial theoretical support for the scaling law of sample complexity in downstream tasks observed in empirical studies.

\subsection{Example: Compositional Space}\label{sec:example}
We now present a concrete example illustrating the application of \zcref{thm:error-bound}, highlighting the advantages of caulking in scenarios where the regression function $f^*$ is a composition of smooth and sparse functions~\citep{schmidt-hieberNonparametricRegressionUsing2020,suzukiDeepLearningAdaptive2021,kohlerRateConvergenceFully2021a}.
In particular, these works consider the case where $f^* \in \mathcal{F}_H = \mathcal{F}_H \circ \cdots \circ \mathcal{F}_1$, with
\begin{align}
  \mathcal{F}_H\!\circ\!...\!\circ\!\mathcal{F}_1 = \Bab{f_H\!\circ\!...\!\circ\!f_1 :  f_i \in \mathcal{F}_i, i = 1, ..., H}, \label{eq:compositional-space}
\end{align}
where each $\mathcal{F}_i$ is a class of smooth and sparse functions, such as the sparse H\"older class~\citep{schmidt-hieberNonparametricRegressionUsing2020,kohlerRateConvergenceFully2021a} or the anisotropic Besov class~\citep{suzukiDeepLearningAdaptive2021}.
For clarity, we focus on the setting of \citet{schmidt-hieberNonparametricRegressionUsing2020}, where $\mathcal{F}_i$ consists of $\beta_i$-H\"older smooth functions, each depending on $t_i$ out of $d_i$ possible variables.
We first review the relevant existing results before demonstrating how \zcref{thm:error-bound} applies in this context.

\paragraph{Error bound without a pre-trained model.}
\citet{schmidt-hieberNonparametricRegressionUsing2020} established that the estimator $f_n$, obtained via empirical risk minimization over the class of sparse ReLU networks, satisfies the following error bound:
\begin{align}
  \mathbb{E}_{Q^n}[E_Q(f_n)] \le C \max_{i=1,...,H}n^{-\frac{2\alpha_i\beta_i}{2\alpha_i\beta_i+t_i}}\mathrm{polylog}(n), \label{eq:error-bound-without-pre-trained-model}
\end{align}
for some constant $C > 0$, where $\alpha_i = \prod_{\ell=i+1}^H\max\Bab{1, \beta_\ell}$.
Here, $\alpha_i$ reflects the (worst-case) H\"older smoothness among the functions in the composition $\mathcal{F}_H\circ\cdots\circ\mathcal{F}_i$.
Classical nonparametric regression theory tells us that the minimax error rate for the class $\mathcal{F}_i$ is $n^{-\frac{2\beta_i}{2\beta_i+t_i}}$.
The rate in \zcref{eq:error-bound-without-pre-trained-model} is thus the slowest among the component classes $\mathcal{F}_i$, with each rate further degraded by the factor $\alpha_i$ due to the compositional structure.

\paragraph{Error bound with a pre-trained model (caulking).}
Now consider the scenario where a pre-trained model $f_{\mathrm{pre}}$ is available, and $f^*$ is $(1/n, \mathcal{F}_{i_h}\circ\cdots\circ\mathcal{F}_{i_e})$-caulkable by $f_{\mathrm{pre}} = (g_{h,m}, g_{e,m})$ for some $1 < i_e < i_h < H$.
The complexity of the class $\mathcal{F}_{i_h}\circ\cdots\circ\mathcal{F}_{i_e}$ can be effectively controlled.
\begin{theorem}[\cite{schmidt-hieberNonparametricRegressionUsing2020}]\label{thm:complexity-of-comp-holder-space-by-relu}
  $\mathcal{F}_{i_h}\circ...\circ\mathcal{F}_{i_e}$ is $\min_{i=i_h,...,i_e}\frac{\alpha_i\beta_i}{\alpha_{i_h}t_i}$-complex in terms of \zcref{def:beta-complexity}.
\end{theorem}
By applying \zcref{thm:error-bound} together with \zcref{thm:complexity-of-comp-holder-space-by-relu}, we obtain the following error bound for empirical caulking:
\begin{align}
  \mathbb{E}_{Q^n}[E_Q(f_n)] \le C \max_{i=i_e,...,i_h}n^{-\frac{2\alpha_i\beta_i}{2\alpha_i\beta_i+t_i}}\mathrm{polylog}(n), \label{eq:error-bound-with-pre-trained-model}
\end{align}
for some constant $C > 0$.
A comparison of \zcref{eq:error-bound-without-pre-trained-model} and \zcref{eq:error-bound-with-pre-trained-model} yields several important insights:
\begin{itemize}
  \item Both bounds take the maximum over $n^{-\frac{2\alpha_i\beta_i}{2\alpha_i\beta_i+t_i}}$, but with a pre-trained model, the range of the index $i$ is restricted to a narrower subset. This demonstrates the improved error rate that can be achieved by leveraging a pre-trained model through empirical caulking.
  \item The result in \zcref{eq:error-bound-with-pre-trained-model} can be interpreted as allowing the learner to bypass the complexities associated with learning the feature extractor part~($\mathcal{F}_{i_e-1},..., \mathcal{F}_1$) and the head part~($\mathcal{F}_{H},..., \mathcal{F}_{i_h+1}$) by utilizing the pre-trained model. Moreover, the greater the complexity of the pre-trained model~(increasing $i_e$ and decreasing $i_h$), the greater the improvement in the error rate. This partially supports the scaling law of the sample complexity of the downstream task observed in empirical studies.
  \item The influence of the feature extractor classes $\mathcal{F}_{i_e-1},..., \mathcal{F}_1$ is entirely eliminated from the error bound. In fact, the error rate in \zcref{eq:error-bound-with-pre-trained-model} remains unchanged even if these classes are replaced by arbitrary alternatives, as long as caulkability is maintained.
  \item The dependence on the head classes $\mathcal{F}_{H},..., \mathcal{F}_{i_h+1}$ persists, but only through the H\"older smoothness parameter $\alpha_i$.
\end{itemize}
More details for this example can be found in \zcref{sec:example-detailed}.

\section{Extension to Classification}\label{sec:classification}
In this section, we present the extension of our theoretical results to binary classification problems.

\paragraph{Setup.}
Let $X \in \mathcal{X}$ and $Y \in \Bab{0,1}$ denote the covariate and label, respectively.
In place of the regression function used in the regression setting, we consider the score function $f^*(x) = \mathbb{P}_Q\Bab{Y=1 | X=x}$.
Accordingly, we assume that the learner has access to a pre-trained model $f_{\mathrm{pre}}$ for the score function $f^*$, constructed from a source sample of size $m$ drawn from $P$. Given the pre-trained model $f_{\mathrm{pre}}$ and a target sample of size $n$ drawn from $Q$, the learner's objective is to construct a classifier $h:\mathcal{X}\to\Bab{0,1}$ that minimizes the classification error on the target distribution, i.e., $\mathbb{E}_Q[\mathds{1}\Bab{h(X) \ne Y}]$.
Let $h_n$ denote the resulting estimated classifier.

The accuracy of a classifier $h$ is evaluated by its excess error.
Let $h^*$ denote the Bayes optimal classifier that minimizes the classification error, which is given by $h^*(x) = \mathds{1}\Bab{f^*(x) > 1/2}$.
The excess error of a classifier $h$ is then defined as
\begin{align}
  E_Q(h) = \mathbb{E}_Q\ab[\mathds{1}\Bab{h(X) \ne Y}] - \mathbb{E}_Q\ab[\mathds{1}\Bab{h^*(X) \ne Y}].
\end{align}

\paragraph{Algorithm.}
We describe the plug-in estimator based on the estimation of the score function $f^*$.
The procedure first estimates the score function $f^*$, denoted by $f_n$, and then constructs the estimated classifier as $h_n(x) = \mathds{1}\Bab{f_n(x) > 1/2}$.
The classification error of the plug-in estimator can be controlled via the regression error bound.
Specifically, we have
\begin{align}
  \mathbb{E}_{Q^n}[E_Q(h_n)] \le 2\sqrt{\mathbb{E}_{Q^n}\ab[\Vab{f^* - f_n}_{L^2(Q_X)}^2]}.
\end{align}
Therefore, we employ empirical caulking with the squared error loss for estimating $f_n$.
Suppose that the ideal score function $f^*$ is $(\epsilon_n, \mathcal{G})$-caulkable by the pre-trained model $f_{\mathrm{pre}} = (g_{h}, g_{e})$.
Given a class $\mathcal{G}_n \subseteq \mathcal{G}$, and letting $\mathcal{F}_n = \Bab{g_h \circ g_a \circ g_e : g_a \in \mathcal{G}_n}$, the estimated score function is given by
\begin{align}
  f_n = \argmin_{f \in \mathcal{F}_n}\frac{1}{n}\sum_{i=1}^n\ab(f(X_i) - Y_i)^2.
\end{align}

\paragraph{Error analysis.}
We establish an analog of \zcref{thm:error-bound} for the classification setting.
\begin{theorem}\label{thm:cls-error-bound}
  Let $\alpha \in (0,1]$ and $\beta > 0$.
  Suppose that $f^*$ is $(n^{-\frac{2\alpha\beta}{2\alpha\beta+1}}, \mathcal{G})$-caulkable by $f_{\mathrm{pre}} = (g_{h}, g_{e})$ for some $\beta$-complex class $\mathcal{G}$, and $g_{h}$ is $\alpha$-H\"older continuous.
  Let $f_n$ denote the estimated score function obtained by empirical caulking with $\mathcal{G}_n = \mathcal{G}_J$, where $\mathcal{G}_J$ is as in \zcref{def:beta-complexity} and $J=\lceil n^{1/(2\alpha\beta+1)} \rceil$.
  Let $h_n(x) = \mathds{1}\Bab{f_n(x) > 1/2}$.
  Under \zcref{asm:regularity}, we have
  \begin{align}
    \mathbb{E}_{Q^n}[E_Q(h_n)] \le C n^{-\frac{\alpha\beta}{2\alpha\beta+1}}\mathrm{polylog}(n),
  \end{align}
  for some constant $C > 0$.
\end{theorem}

\section{Proof Sketch}\label{sec:proof-sketch}
In this section, we provide an outline of the proof for the main result, \zcref{thm:error-bound}.
Complete proofs of \zcref{thm:error-bound} and other omitted results are given in the supplemental material.
The proof of \zcref{thm:error-bound} relies on the following error bound, which depends on both the approximation capability and the covering number of the class $\mathcal{F}_n$.
\begin{theorem}[\citet{schmidt-hieberNonparametricRegressionUsing2020,hayakawaMinimaxOptimalitySuperiority2020}]\label{thm:error-bound-general}
  Under the same assumptions as \zcref{thm:error-bound}, we have
  \begin{align}
    \mathbb{E}_{Q^n}\ab[\Vab{f_n - f^*}_{L^2(Q_X)}^2] \le C\pab{\inf_{f \in \mathcal{F}_n}\Vab{f - f^*}_{L^2(Q_X)}^2 + \frac{\ln\ab(N\ab(n^{-1}, \mathcal{F}_n, \Vab{\cdot}_{L^\infty(Q_X)}))}{n}},
  \end{align}
  for some constant $C > 0$ depending on $\sigma_\xi$ and $\Delta_\Omega$ in \zcref{asm:regularity}.
\end{theorem}
To utilize \zcref{thm:error-bound-general}, it is necessary to control both the approximation capability and the covering number of $\mathcal{F}_n$ in terms of the $\beta$-complexity of the class $\mathcal{G}_N$. For this purpose, we establish the following two propositions.
\begin{proposition}\label{prop:composition-cv-bound}
  Given a class of functions $\mathcal{G} = \Bab{g: \mathcal{Z}_i\to\mathcal{Z}_o}$, define $\mathcal{F} = \Bab*{g_h \circ g_a \circ g_e: g_a \in \mathcal{G}}$ for some functions $g_h:\mathcal{Z}_o\to\mathbb{R}$ and $g_e:\mathcal{X}\to\mathcal{Z}_i$. Assume that $g_h$ is $\alpha$-H\"older continuous with respect to a norm $\Vab{\cdot}$ on $\mathcal{Z}_o$ for some $\alpha \in (0,1]$. Then, for any measure $\nu$ on $\mathcal{X}$ and any $\delta > 0$, we have
  \begin{align}
    N(\delta, \mathcal{F}, \Vab{\cdot}_{L^\infty(\nu)}) \le N\ab(\ab(\frac{\delta}{C_\alpha}\ab)^{1/\alpha}, \mathcal{G}, \Vab{\cdot}_{L^\infty}),
  \end{align}
  where $C_\alpha$ is a constant for the H\"older continuity of $g_h$.
\end{proposition}
\begin{proposition}\label{prop:approximation-bound}
  Given a class of functions $\mathcal{G} = \Bab{g: \mathcal{Z}_i\to\mathcal{Z}_o}$, define $\mathcal{F} = \Bab*{g_h \circ g_a \circ g_e: g_a \in \mathcal{G}}$ for some functions $g_h:\mathcal{Z}_o\to\mathbb{R}$ and $g_e:\mathcal{X}\to\mathcal{Z}_i$. Assume that $g_h$ is $\alpha$-H\"older continuous with respect to a norm $\Vab{\cdot}$ on $\mathcal{Z}_o$ for some $\alpha \in (0,1]$. For any $f^* = g_h \circ g^*_a \circ g_e$ with $g^*_a$ possibly not in $\mathcal{G}$, we have
  \begin{align}
    \inf_{f \in \mathcal{F}}\Vab{f - f^*}_{L^2(Q_X)} \le C_\alpha\inf_{g_a \in \mathcal{G}}\Vab{g_a - g^*_a}_{L^\infty}^\alpha,
  \end{align}
  where $C_\alpha$ is a constant for the H\"older continuity of $g_h$.
\end{proposition}
By applying \zcref{prop:composition-cv-bound} and \zcref{prop:approximation-bound} to the first and second terms of \zcref{thm:error-bound-general}, and using the $\beta$-complexity of the class $\mathcal{G}_J$, we obtain $\mathbb{E}_{Q^n}\ab[\Vab{f_n - f^*}_{L^2(Q_X)}^2] = O(J^{-2\alpha\beta} + J/n)$.
Choosing $J$ as in the statement yields the desired result.

\section{Empirical Evaluation of Caulking}\label{sec:experiments}

We conduct two types of experiments to support our theoretical claims.

\subsection{Fine-tuning of CNNs}

We first fine-tune ResNet-50 (26M parameters, \cite{he2016deep}) and Wide ResNet-50-2 (68M parameters, \cite{zagoruyko2016wide}) with adapters on the clipart domain of the Office-Home dataset~\citep{venkateswaraDeepHashingNetwork2017} after pre-training on other domains.
\zcref{fig:experiments_clipart} shows the relationship between the error rates and the depth of adapters, where the adapters consist of linear layers forming MLPs.
Due to the high similarity of the model architectures between ResNet and Wide ResNet, the results can suggest that the larger model achieves higher performance with a single-layer adapter, whereas the smaller model requires a more complex adapter.
These findings are consistent with our theoretical results, which suggest that a larger pre-trained model can be effectively adapted with a simpler adapter.

\begin{figure}[t]
  \centering
  \begin{minipage}[t]{.48\textwidth}
    \centering
    \includegraphics[width=\linewidth]{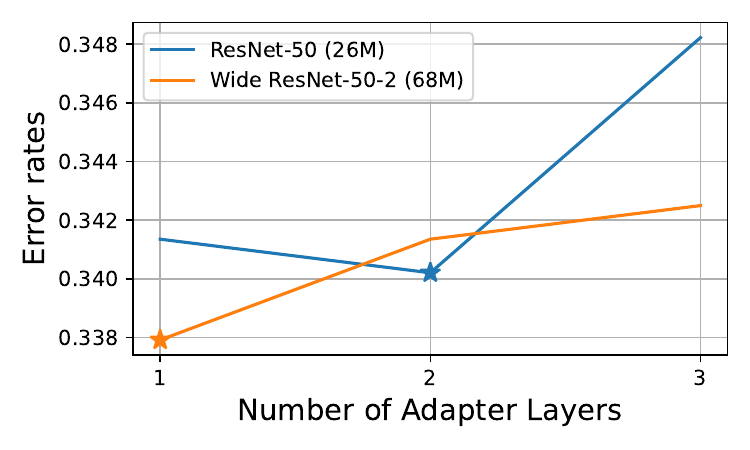}
    \caption{The relationship between the depth of adapters and the error rate on the target domain. Minimum error rates for each model are marked by $\star$.}
    \label{fig:experiments_clipart}
  \end{minipage}
  \hfill
  \begin{minipage}[t]{.48\textwidth}
    \centering
    \includegraphics[width=\linewidth]{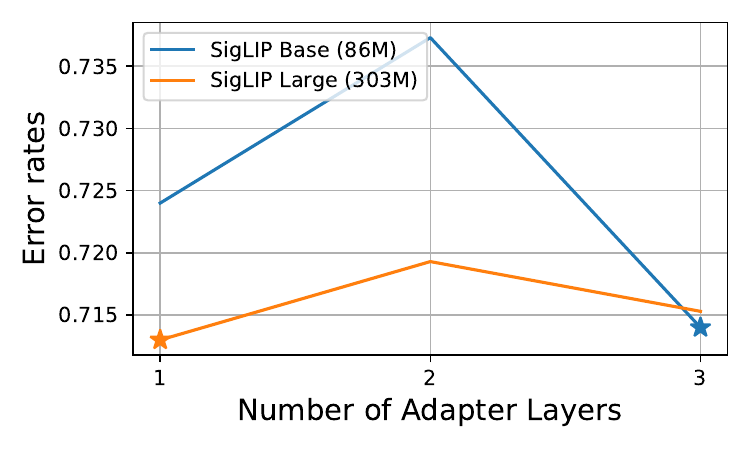}
    \caption{The relationship between the depth of adapters and the error rate on the MMStar dataset~\citep{chen2024are}.}
    \label{fig:experiments_vlm}
  \end{minipage}
\end{figure}

\subsection{Integrating Vision Capabilities to LLMs}

We then incorporate vision capabilities into a pre-trained language model by integrating pre-trained vision encoders and adapters.
Specifically, we use a Llama-3-style Transformer model (135M parameters, \cite{allal2024smollm}) together with SigLIP visual feature extractors of varying sizes (86M and 303M parameters, \cite{tschannen2025siglip2}) and train them as Vision Language Models following the training recipe of \cite{gosthipaty2024nanovlm} with a slight modification.
\zcref{fig:experiments_vlm} illustrates the relationship between the adapter depth and error rates.
Once again, the results are consistent with our theoretical findings.

Further experimental details can be  found in \zcref{app:sec:experimental_details}.

\section{Conclusion}
This paper presents a theoretical analysis of the advantages of utilizing pre-trained models for downstream tasks, with particular emphasis on the scaling laws associated with pre-trained models. These scaling laws indicate that larger pre-trained models can significantly reduce the sample complexity required for downstream tasks. We rigorously explain this phenomenon through the novel concept of {\em caulkability}. Our experimental results further corroborate our theoretical findings: in particular, larger pre-trained models can be effectively adapted to downstream tasks using simpler adapters. These insights clarify the benefits of employing larger pre-trained models for downstream applications.

\paragraph{Open question.}
The primary focus of this work is to identify the properties of pre-trained models that underlie the scaling laws observed in practice, and we show that caulkability by an adapter model class of decreasing complexity is a sufficient condition for these scaling laws. However, we do not propose a concrete learning algorithm for constructing a pre-trained model that exhibits the caulkability property. Designing such a learning algorithm remains an important and intriguing open question.

\subsubsection*{Acknowledgments}
This work was partly supported by JSPS KAKENHI Grant Numbers JP23K13011 to K.F., JP23K28146 and JP24K20836 to K.M, and JST BOOST Grant Number JPMJBY24G2 to R.H.

\bibliographystyle{apalike}
\bibliography{references}

\appendix

\section{Missing Proofs}\label{sec:missing-proofs}
\subsection{Proof of \zcref{thm:error-bound}}
As shown in \zcref{sec:proof-sketch}, the proof of \zcref{thm:error-bound} follows a similar approach to that of \citet{schmidt-hieberNonparametricRegressionUsing2020,hayakawaMinimaxOptimalitySuperiority2020}, utilizing \zcref{thm:error-bound-general}. For clarity, we present here a more rigorous statement of \zcref{thm:error-bound-general}:
\begin{theorem}\label{thm:error-bound-general-A}
  Let $\sigma > 0$ and $F > 0$ be constants. Let $X$, $Y$, and $\xi$ be random variables defined on a probability space $(\mathcal{Z}, \mathfrak{Z}, \nu)$, and let $\nu_X$ be the marginal distribution of $X$. Given a measurable function $f^*:\mathcal{X}\to\mathbb{R}$ such that $\Vab{f^*}_{L^\infty(\nu_X)} \le F$, suppose that these variables follow the regression model $Y = f^*(X) + \xi$. Suppose also that $\xi$ is independent of $X$ and sub-Gaussian with variance-proxy $\sigma^2$, i.e., $\mathbb{E}_{\nu}\ab[e^{\lambda \xi}] \le e^{\lambda^2\sigma^2/2}$ for all $\lambda \in \mathbb{R}$. Let $\mathcal{F}$ be a class of measurable functions $f:\mathcal{X}\to\mathbb{R}$ such that $\sup_{f \in \mathcal{F}} \Vab{f}_{L^\infty(\nu_X)} \le F$. Let $(X_1, Y_1), ..., (X_n, Y_n)$ be $n$ i.i.d. copies of $(X, Y)$, and define $f_n \in \mathcal{F}$ as a function that attains the following infimum:
  \begin{align}
    \inf_{f \in \mathcal{F}}\frac{1}{n}\sum_{i=1}^n\ab(f(X_i) - Y_i)^2.
  \end{align}
  Then, for any $\delta > 0$ satisfying $\ln(N(\delta, \mathcal{F}, \Vab{\cdot}_{L^\infty(\nu_X)})) \ge 1$, we have
  \begin{align}
    \mathbb{E}_{\nu^n}\ab[\Vab{f_n - f^*}_{L^2(\nu_X)}^2] \le C\ab(\inf_{f \in \mathcal{F}}\Vab{f - f^*}_{L^2(\nu_X)}^2 + \frac{\ab(F^2 + \sigma^2)\ln\ab(N\ab(\delta, \mathcal{F}, \Vab{\cdot}_{L^\infty(\nu)}))}{n} + (F+\sigma)\delta),
  \end{align}
  for some universal constant $C > 0$.
\end{theorem}
The proof of \zcref{thm:error-bound-general-A} is given in \zcref{sec:proof-of-thm:error-bound-general-A}. \zcref{thm:error-bound-general-A} characterizes the error bound of the least-squares estimator $f_n$ over the class $\mathcal{F}$ in terms of the approximation error of $f^*$ by $\mathcal{F}$ and the covering number of $\mathcal{F}$. To control these two terms, we utilize \zcref{prop:composition-cv-bound} and \zcref{prop:approximation-bound}. Recall the statements of \zcref{prop:composition-cv-bound} and \zcref{prop:approximation-bound}:
\begin{proposition}\label{prop:composition-cv-bound-A}
  Given a class of functions $\mathcal{G} = \Bab{g: \mathcal{Z}_i\to\mathcal{Z}_o}$, define $\mathcal{F} = \Bab*{g_h \circ g_a \circ g_e: g_a \in \mathcal{G}}$ for some functions $g_h:\mathcal{Z}_o\to\mathbb{R}$ and $g_e:\mathcal{X}\to\mathcal{Z}_i$. Assume that $g_h$ is $\alpha$-H\"older continuous with respect to a norm $\Vab{\cdot}$ on $\mathcal{Z}_o$ for some $\alpha \in (0,1]$. Then, for any measure $\nu$ on $\mathcal{X}$ and any $\delta > 0$, we have
  \begin{align}
    N(\delta, \mathcal{F}, \Vab{\cdot}_{L^\infty(\nu)}) \le N\ab(\ab(\frac{\delta}{C_\alpha}\ab)^{1/\alpha}, \mathcal{G}, \Vab{\cdot}_{L^\infty}),
  \end{align}
  where $C_\alpha$ is a constant for the H\"older continuity of $g_h$.
\end{proposition}
\begin{proposition}\label{prop:approximation-bound-A}
  Given a class of functions $\mathcal{G} = \Bab{g: \mathcal{Z}_i\to\mathcal{Z}_o}$, define $\mathcal{F} = \Bab*{g_h \circ g_a \circ g_e: g_a \in \mathcal{G}}$ for some functions $g_h:\mathcal{Z}_o\to\mathbb{R}$ and $g_e:\mathcal{X}\to\mathcal{Z}_i$. Assume that $g_h$ is $\alpha$-H\"older continuous with respect to a norm $\Vab{\cdot}$ on $\mathcal{Z}_o$ for some $\alpha \in (0,1]$. For any $f^* = g_h \circ g^*_a \circ g_e$ with $g^*_a$ possibly not in $\mathcal{G}$, we have
  \begin{align}
    \inf_{f \in \mathcal{F}}\Vab{f - f^*}_{L^2(Q_X)} \le C_\alpha\inf_{g_a \in \mathcal{G}}\Vab{g_a - g^*_a}_{L^\infty}^\alpha,
  \end{align}
  where $C_\alpha$ is a constant for the H\"older continuity of $g_h$.
\end{proposition}
The proofs of \zcref{prop:composition-cv-bound-A} and \zcref{prop:approximation-bound-A} are given in \zcref{sec:proof-of-prop:composition-cv-bound-A} and \zcref{sec:proof-of-prop:approximation-bound-A}, respectively. Building upon \zcref{thm:error-bound-general-A}, \zcref{prop:composition-cv-bound-A}, and \zcref{prop:approximation-bound-A}, we can prove \zcref{thm:error-bound}:
\begin{proof}[Proof of \zcref{thm:error-bound}]
  We apply \zcref{thm:error-bound-general-A} with $\mathcal{F} = \mathcal{F}_n$, where $\mathcal{F}_n = \Bab{g_h \circ g_a \circ g_e : g_a \in \mathcal{G}_n}$ for a sequence of classes $\mathcal{G}_n$ in \zcref{def:beta-complexity} corresponding to a $\beta$-complex class $\mathcal{G}$. With $\delta = 1/n$, we have
  \begin{multline}
    \mathbb{E}_{\nu^n}\ab[\Vab{f_n - f^*}_{L^2(\nu_X)}^2] \le \\ C\ab(\inf_{f \in \mathcal{F}_n}\Vab{f - f^*}_{L^2(\nu_X)}^2 + \frac{\ab(F^2 + \sigma^2)\ln\ab(N\ab(n^{-1}, \mathcal{F}_n, \Vab{\cdot}_{L^\infty(\nu)}))}{n} + \frac{(F+\sigma)}{n}).
  \end{multline}
  Since $\mathcal{G}_n$ is a sequence of classes in \zcref{def:beta-complexity} corresponding to a $\beta$-complex class $\mathcal{G}$, by \zcref{prop:composition-cv-bound-A,prop:approximation-bound-A}, we have
  \begin{align}
    N(n^{-1}, \mathcal{F}_n, \Vab{\cdot}_{L^\infty(\nu)}) \le N\ab(\ab(\frac{n^{-1}}{C_\alpha}\ab)^{1/\alpha}, \mathcal{G}_n, \Vab{\cdot}_{L^\infty}) \le cJ\mathrm{polylog}(J, (C_\alpha n)^{-1/\alpha}),
  \end{align}
  and
  \begin{align}
    \inf_{f \in \mathcal{F}_n}\Vab{f - f^*}_{L^2(\nu_X)}^2 \le C_\alpha\inf_{g_a \in \mathcal{G}_n}\Vab{g_a - g^*_a}_{L^\infty}^\alpha \le c'J^{-2\alpha\beta},
  \end{align}
  where $J$ is defined in the statement. Hence, we have
  \begin{align}
    \mathbb{E}_{\nu^n}\ab[\Vab{f_n - f^*}_{L^2(\nu_X)}^2] \le C\ab(J^{-2\alpha\beta} + \frac{J\mathrm{polylog}(J, (C_\alpha n)^{-1/\alpha})}{n} + \frac{1}{n}). \label{eq:error-bound-J}
  \end{align}
  \zcref{eq:error-bound-J} is minimized with $J$ defined in the statement, which yields the desired claim.
\end{proof}

\subsection{Proof of \zcref{thm:error-bound-general-A}}\label{sec:proof-of-thm:error-bound-general-A}
We begin by introducing notation that will be used throughout the proof for clarity and conciseness. Let $\xi_1, ..., \xi_n$ be i.i.d. copies of $\xi$. For any function $f:\mathcal{X}\to\mathbb{R}$, define the empirical $L^2$ norm of $f$ and noises $\xi_i$ as
\begin{align}
  \Vab{f}^2_n := \frac{1}{n} \sum_{i=1}^n f^2(X_i), \quad \Vab{\xi}^2_n := \frac{1}{n} \sum_{i=1}^n \xi_i^2,
\end{align}
and the empirical inner product with the noise variables as
\begin{align}
  \ab<f,\xi>_n := \frac{1}{n} \sum_{i=1}^n \xi_i f(X_i).
\end{align}
For the notational convenience, we use $\mathbb{V}_\nu[X]$ to denote the variance of a random variable $X$ defined on a probability space $(\mathcal{Z}, \mathfrak{Z}, \nu)$.

The proof consists of two steps. In the first step, we bound the expected error of $f_n$ by its empirical error. In the second step, we give a bound on the empirical error of $f_n$. Specifically, we show the following two lemmas.
\begin{lemma}\label{lem:bound-pop-error}
  Under the same assumptions as in \zcref{thm:error-bound-general-A}, we have
  \begin{multline}
    \mathbb{E}_{\nu^n}\ab[\Vab{f_n - f^*}^2_{L^2(\nu_X)}] \le \\ \inf_{\zeta > 0}\pab[Bigg]{\ab(1+\zeta)\mathbb{E}_{\nu^n}\ab[\Vab{f_n - f^*}_n^2] \\ + \ab(2 + 3\zeta + \frac{1}{2\zeta})\frac{512F^2\ln(N_\delta)}{n} + \ab(1 + \zeta + \frac{1}{4\zeta})\frac{9F^2}{n} + 4(2+3\zeta)F\delta}.
  \end{multline}
  where $N_\delta = N(\delta, \mathcal{F}, \Vab{\cdot}_{L^\infty(\nu_X)})$.
\end{lemma}
\begin{lemma}\label{lem:bound-emp-error}
  Under the same assumptions as in \zcref{thm:error-bound-general-A}, we have
  \begin{multline}
    \mathbb{E}_{\nu^n}\ab[\Vab{f_n - f^*}_n^2] \le \inf_{\zeta > 0}\pab[Bigg]{\ab(1+\zeta)\inf_{f \in \mathcal{F}}\Vab{f - f^*}_{L^2(\nu_X)}^2 \\ + \ab(1 + \frac{\zeta}{2} + \frac{1}{2\zeta})\frac{8\sigma^2\ln\ab(3N_\delta)}{n} + \ab((1+\zeta)\sqrt{2\pi}\sigma + 4F\zeta)\delta}.
  \end{multline}
  where $N_\delta = N(\delta, \mathcal{F}, \Vab{\cdot}_{L^\infty(\nu_X)})$.
\end{lemma}
Combining \zcref{lem:bound-pop-error} and \zcref{lem:bound-emp-error} immediately yields \zcref{thm:error-bound-general-A}.
\begin{proof}[Proof of \zcref{thm:error-bound-general-A}]
  Combining \zcref{lem:bound-pop-error} and \zcref{lem:bound-emp-error} and taking $\zeta$ in these lemmas as universal constants, we obtain
  \begin{align}
    \mathbb{E}_{\nu^n}\ab[\Vab{f_n - f^*}^2_{L^2(\nu_X)}] \le C\ab(\inf_{f \in \mathcal{F}}\Vab{f - f^*}_{L^2(\nu_X)}^2 + \frac{\ab(\sigma^2 + F^2)\ln(N_\delta)}{n} + \ab(\sigma + F)\delta),
  \end{align}
  for some universal constant $C > 0$.
\end{proof}

In the following subsections, we present the proofs of \zcref{lem:bound-pop-error} and \zcref{lem:bound-emp-error}. The bounds in both lemmas are determined by the covering number of the class $\mathcal{F}$. For notational convenience, let $f_k$ denote the centers of balls in a $\delta$-covering of $\mathcal{F}_n$ with respect to the norm $\Vab{\cdot}_{L^\infty(\nu_X)}$. Define $g_k = f_k - f^*$, and let $K = \argmin_{k} \Vab{f_n - f_k}_{L^\infty(\nu_X)}$.

\subsubsection{Proof of \zcref{lem:bound-pop-error}}
\begin{proof}[Proof of \zcref{lem:bound-pop-error}]
  Let $\tau > 0$ be a constant to be specified later. We have
  \begin{align}
    & \mathbb{E}_{\nu^n}\ab[\Vab{f_n - f^*}^2_{L^2(\nu_X)}] \\
    =& \mathbb{E}_{\nu^n}\ab[\Vab{f_n - f^*}_n^2] + \mathbb{E}_{\nu^n}\ab[\Vab{f_n - f^*}^2_{L^2(\nu_X)} - \Vab{f_n - f^*}_n^2] \\
    \le& \mathbb{E}_{\nu^n}\ab[\Vab{f_n - f^*}_n^2] + 8F\delta + \mathbb{E}_{\nu^n}\ab[\Vab{g_K}^2_{L^2(\nu_X)} - \Vab{g_K}_n^2] \\
    \le& \mathbb{E}_{\nu^n}\ab[\Vab{f_n - f^*}_n^2] + 8F\delta + \mathbb{E}_{\nu^n}\ab[\max\Bab{\tau^2,\Vab{g_K}^2_{L^2(\nu_X)}}]^{1/2}\mathbb{E}_{\nu^n}\ab[\frac{\ab(\Vab{g_K}^2_{L^2(\nu_X)} - \Vab{g_K}_n^2)^2}{\max\Bab{\tau^2,\Vab{g_K}^2_{L^2(\nu_X)}}}]^{1/2}, \label{eq:bound-pop-error-1}
  \end{align}
  where the third inequality follows from the Cauchy-Schwarz inequality, and in the second inequality we use the following:
  \begin{align}
    & \Vab{f_n - f^*}_{L^2(\nu_X)}^2 - \Vab{f_n - f^*}_n^2 \\
    =&
    \begin{multlined}[t][.95\textwidth]
      \Vab{g_K}_{L^2(\nu_X)}^2 + \mathbb{E}_\nu\ab[\ab(f_K(X) - f_n(X))\ab(f_K(X) + f_n(X) - 2f^*(X))] \\ - \Vab{g_K}_n^2 - \frac{1}{n}\sum_{i=1}^n\ab(f_K(X_i) - f_n(X_i))\ab(f_K(X_i) + f_n(X_i) - 2f^*(X_i))
    \end{multlined} \\
    \le& \Vab{g_K}^2 - \Vab{g_K}_n^2 + 8F\delta,
  \end{align}
  almost surely.

  For the last term in \zcref{eq:bound-pop-error-1}, we have
  \begin{align}
    &\mathbb{E}_{\nu^n}\ab[\max\Bab{\tau^2,\Vab{g_K}^2_{L^2(\nu_X)}}] \\
    \le& \mathbb{E}_{\nu^n}\ab[\Vab{f_n - f^*}^2_{L^2(\nu_X)}] + \mathbb{E}_{\nu^n}\ab[\Vab{g_K}^2_{L^2(\nu_X)} - \Vab{f_n - f^*}^2_{L^2(\nu_X)}] + \tau^2 \\
    =& \mathbb{E}_{\nu^n}\ab[\Vab{f_n - f^*}^2_{L^2(\nu_X)}] + \mathbb{E}_{\nu^n}\ab[\mathbb{E}_{\nu}\ab[\ab(f_K(X) - f_n(X))\ab(f_K + f_n - 2f^*)]] + \tau^2 \\
    \le& \mathbb{E}_{\nu^n}\ab[\Vab{f_n - f^*}^2_{L^2(\nu_X)}] + 4F\delta + \tau^2. \label{eq:bound-pop-error-2}
  \end{align}
  Furthermore,
  \begin{align}
    \mathbb{E}_{\nu^n}\ab[\frac{\ab(\Vab{g_K}^2_{L^2(\nu_X)} - \Vab{g_K}_n^2)^2}{\max\Bab{\tau^2,\Vab{g_K}^2_{L^2(\nu_X)}}}]
    = \mathbb{E}_{\nu^n}\ab[\ab(\frac{1}{n}\sum_{i=1}^n\frac{\ab(\Vab{g_K}_{L^2(\nu_X)}^2 - g^2_K(X_i))}{\max\Bab{\tau^2,\Vab{g_K}_{L^2(\nu_X)}}})^2].
  \end{align}
  Define $Z_{k,i} = \frac{(g^2_k(X_i) - \Vab{g_k}_{L^2(\nu_X)}^2)}{\max\Bab{\tau^2,\Vab{g_k}_{L^2(\nu_X)}}}$. Then $Z_{k,i}$ is zero-mean and
  \begin{align}
    \mathbb{E}_\nu\ab[Z_{k,i}^2] =& \mathbb{E}_\nu\ab[\frac{(g_k(X_i) - g_k(X'_i))^2(g_k(X_i) + g_k(X'_i))^2}{\tau^2\lor\Vab{g_k}_{L^2(\nu_X)}^2}] \\
    \le& 8\Vab{g_k}_{L^\infty(\nu_X)}^2\frac{\mathbb{V}_\nu\ab[g_k(X_i)]}{\max\Bab{\tau^2,\Vab{g_k}_{L^2(\nu_X)}^2}} \le 8\Vab{g_k}_{L^\infty(\nu_X)}^2 \le 32F^2,
  \end{align}
  where $X'_i$ are independent copies of $X_i$, and $\ab|Z_{k,i}| \le \Vab{g_k}_{L^\infty(\nu_X)}^2/\tau \le 4F^2/\tau$ almost surely. By the Bernstein condition for bounded random variables, for all $k$ and $\lambda \in (0, 3\tau/4F^2)$,
  \begin{align}
    \mathbb{E}_{\nu^n}\ab[\exp\ab(\lambda \sum_{i=1}^n Z_{k,i})] \le \exp\ab(\frac{32n\lambda^2F^2}{2(1 - 4\lambda F^2/3\tau)}).
  \end{align}
  Applying \zcref{lem:maximal-finite-sum-squares}, for any $\gamma > 0$,
  \begin{align}
    \mathbb{E}_{\nu^n}\ab[\max_k \ab(\sum_{i=1}^n Z_{k,i})^2] \le 512\ab(nF^2\ln(N_\delta/\gamma) + \gamma\max\Bab{nF^2,\frac{F^4}{9\tau^2}\ln(N_\delta/\gamma)} + \frac{\gamma F^4}{9\tau^2}).
  \end{align}
  Therefore,
  \begin{align}
    \mathbb{E}_{\nu^n}\ab[\frac{\ab(\Vab{g_K}^2_{L^2(\nu_X)} - \Vab{g_K}_n^2)^2}{\max\Bab{\tau^2,\Vab{g_K}^2_{L^2(\nu_X)}}}]
    \le& \frac{1}{n^2}\mathbb{E}_{\nu^n}\ab[\max_k \ab(\sum_{i=1}^n Z_{k,i})^2] \\
    \le& \frac{512F^2}{n}\ab(\ln(N_\delta/\gamma) + \gamma\max\Bab{1,\frac{F^2}{9\tau^2n}\ln(N_\delta/\gamma)} + \frac{\gamma F^2}{9\tau^2n}). \label{eq:bound-pop-error-3}
  \end{align}
  Substituting \zcref{eq:bound-pop-error-2} and \zcref{eq:bound-pop-error-3} into \zcref{eq:bound-pop-error-1}, we obtain
  \begin{multline}
    \mathbb{E}_{\nu^n}\ab[\Vab{f_n - f^*}^2_{L^2(\nu_X)}] \le \mathbb{E}_{\nu^n}\ab[\Vab{f_n - f^*}_n^2] + 8F\delta \\ + \ab(\mathbb{E}_{\nu^n}\ab[\Vab{f_n - f^*}^2_{L^2(\nu_X)}] + 4F\delta + \tau^2)^{1/2}\\\cdot\ab(\frac{512F^2}{n}\ab(\ln(N_\delta/\gamma) + \gamma\max\Bab{1,\frac{F^2}{9\tau^2n}\ln(N_\delta/\gamma)} + \frac{\gamma F^2}{9\tau^2n}))^{1/2}. \label{eq:bound-pop-error-4}
  \end{multline}

  Now, for $a,b,c > 0$, if $x^2 \le a + b\sqrt{x^2 + c^2}$, then $x^2 \le a$ or
  \begin{align}
    x^4 - 2x^2\ab(a + \frac{b^2}{2}) + \ab(a^2 - b^2c^2) \le 0,
  \end{align}
  which implies that
  \begin{align}
    x^2 \le a + \frac{b^2}{2} + b\sqrt{a + \frac{b^2}{4} + c^2}. \label{eq:quadratic-inequality}
  \end{align}

  Therefore, from \zcref{eq:bound-pop-error-4}, we have
  \begin{multline}
    \mathbb{E}_{\nu^n}\ab[\Vab{f_n - f^*}^2_{L^2(\nu_X)}] \le \mathbb{E}_{\nu^n}\ab[\Vab{f_n - f^*}_n^2] + 8F\delta \\ + \frac{512F^2}{n}\ab(\ln(N_\delta/\gamma) + \gamma\max\Bab{1,\frac{F^2}{9\tau^2n}\ln(N_\delta/\gamma)} + \frac{\gamma F^2}{9\tau^2n}) \\ + \ab(\frac{512F^2}{n}\ab(\ln(N_\delta/\gamma) + \gamma\max\Bab{1,\frac{F^2}{9\tau^2n}\ln(N_\delta/\gamma)} + \frac{\gamma F^2}{9\tau^2n}))^{1/2}\\\cdot\ab(\mathbb{E}_{\nu^n}\ab[\Vab{f_n - f^*}_n^2] + \tau^2 + \frac{512F^2}{n}\ab(\ln(N_\delta/\gamma) + \gamma\max\Bab{1,\frac{F^2}{9\tau^2n}\ln(N_\delta/\gamma)} + \frac{\gamma F^2}{9\tau^2n}) + 12F\delta)^{1/2}.
  \end{multline}
  By the AM-GM inequality, for any $\zeta > 0$, we have
  \begin{multline}
    \mathbb{E}_{\nu^n}\ab[\Vab{f_n - f^*}^2_{L^2(\nu_X)}] \le \ab(1+\zeta)\mathbb{E}_{\nu^n}\ab[\Vab{f_n - f^*}_n^2] + 4(2+3\zeta)F\delta \\ + \ab(1 + \zeta + \frac{1}{4\zeta})\frac{512F^2}{n}\ab(\ln(N_\delta/\gamma) + \gamma\max\Bab{1,\frac{F^2}{9\tau^2n}\ln(N_\delta/\gamma)} + \frac{\gamma F^2}{9\tau^2n}) + \zeta\tau^2.
  \end{multline}

  Now set
  \begin{align}
    \tau^2 = \frac{512F^2\ln(N_\delta/\gamma)}{n}.
  \end{align}
  Then,
  \begin{multline}
    \mathbb{E}_{\nu^n}\ab[\Vab{f_n - f^*}^2_{L^2(\nu_X)}] \le \ab(1+\zeta)\mathbb{E}_{\nu^n}\ab[\Vab{f_n - f^*}_n^2] + 4(2+3\zeta)F\delta \\ + \ab(1 + \zeta + \frac{1}{4\zeta})\frac{512F^2}{n}\ab(\ln(N_\delta/\gamma) + \gamma\ab(1 + \frac{9}{512\ln(N_\delta/\gamma)})) + \frac{512\zeta F^2\ln(N_\delta/\gamma)}{n}.
  \end{multline}
  Simplifying, we obtain
  \begin{multline}
    \mathbb{E}_{\nu^n}\ab[\Vab{f_n - f^*}^2_{L^2(\nu_X)}] \le \ab(1+\zeta)\mathbb{E}_{\nu^n}\ab[\Vab{f_n - f^*}_n^2] + 4(2+3\zeta)F\delta \\ + \ab(1 + 2\zeta + \frac{1}{4\zeta})\frac{512F^2\ln(N_\delta/\gamma)}{n} + \ab(1 + \zeta + \frac{1}{4\zeta})\frac{512F^2\gamma}{n}\ab(1 + \frac{9}{512\ln(N_\delta/\gamma)}).
  \end{multline}
  Setting $\gamma = \ln(N_\delta)$, we have
  \begin{multline}
    \mathbb{E}_{\nu^n}\ab[\Vab{f_n - f^*}^2_{L^2(\nu_X)}] \le \ab(1+\zeta)\mathbb{E}_{\nu^n}\ab[\Vab{f_n - f^*}_n^2] + 4(2+3\zeta)F\delta \\ + \ab(2 + 3\zeta + \frac{1}{2\zeta})\frac{512F^2\ln(N_\delta)}{n} + \ab(1 + \zeta + \frac{1}{4\zeta})\frac{9F^2}{n}.
  \end{multline}
  The arbitrariness of $f \in \mathcal{F}$ and $\zeta$ gives the desired result.
\end{proof}

\subsubsection{Proof of \zcref{lem:bound-emp-error}}
\begin{proof}[Proof of \zcref{lem:bound-emp-error}]
  Starting from the regression model $Y = f^*(X) + \xi$, we obtain
  \begin{align}
    \frac{1}{n}\sum_{i=1}^n\ab(f_n(X_i) - Y_i)^2 = \Vab{f_n - f^*}^2_{n} - 2\ab<f_n - f^*,\xi>_n + \Vab{\xi}^2_n.
  \end{align}
  Since $\Vab{\xi}^2_n$ is independent of $f_n$, for any $f \in \mathcal{F}_n$ that is independent of the sample, it follows that
  \begin{align}
    \Vab{f_n - f^*}_n^2 - 2\ab<f_n - f^*,\xi>_n \le \Vab{f - f^*}^2_{n} - 2\ab<f - f^*,\xi>_n.
  \end{align}
  Therefore,
  \begin{align}
    \mathbb{E}_{\nu^n}\ab[\Vab{f_n - f^*}_n^2] =& \mathbb{E}_{\nu^n}\ab[\Vab{f_n - f^*}_n^2 - 2\ab<f_n - f^*,\xi>_n + 2\ab<f_n - f^*,\xi>_n] \\
    \le& \Vab{f - f^*}_{L^2(\nu_X)}^2 + 2\mathbb{E}_{\nu^n}\ab[\ab<f_n - f_K,\xi>_n] + 2\mathbb{E}_{\nu^n}\ab[\ab<g_K,\xi>_n] \\
    \le& \Vab{f - f^*}_{L^2(\nu_X)}^2 + 2\delta\mathbb{E}_{\nu}\ab[\ab|\xi|] + 2\mathbb{E}_{\nu^n}\ab[\Vab{g_K}_n^2]^{1/2}\mathbb{E}_{\nu^n}\ab[\ab|\ab<\frac{g_K}{\Vab{g_K}_n},\xi>_n|^2]^{1/2}. \label{eq:bound-emp-error-1}
  \end{align}

  For the second term in \zcref{eq:bound-emp-error-1}, the subgaussianity of $\xi$ together with Markov's inequality yields
  \begin{align}
    \mathbb{E}_{\nu}\ab[\ab|\xi|] =& \int_0^\infty \ab(\mathbb{P}_{\nu}\Bab{\xi > t} + \mathbb{P}_{\nu}\Bab{-\xi > t}) dt \\
    \le& 2\int_0^\infty \inf_{\lambda > 0}\exp\ab(\frac{\lambda^2\sigma^2}{2} - \lambda t) dt \\
    =& 2\int_0^\infty \exp\ab(-\frac{t^2}{2\sigma^2}) dt \\
    =& \sqrt{2\pi}\sigma. \label{eq:bound-emp-error-2}
  \end{align}

  For the last term in \zcref{eq:bound-emp-error-1}, by the triangle inequality, we have
  \begin{align}
    \Vab{g_K}_n^2 =& \Vab{f_n - f^*}_n^2 + \Vab{g_K}_n^2 - \Vab{f_n - f^*}_n^2 \\
    \le& \Vab{f_n - f_K}_n^2 + \frac{1}{n}\sum_{i=1}^n \ab(f_K(X_i) - f_n(X_i))\ab(f_K(X_i) + f_n(X_i) - 2f^*(X_i)) \\
    \le& \Vab{f_n - f^*}_n^2 + 4F\delta, \label{eq:bound-emp-error-3}
  \end{align}
  almost surely.

  Conditioned on $X_1, ..., X_n$, $Z_k = \ab<\frac{g_k}{\Vab{g_k}_n},\xi>_n$ is zero-mean and sub-Gaussian with variance proxy $\sigma^2/n$. Thus,
  \begin{align}
    \mathbb{E}_{\nu^n}\ab[\exp\ab(\lambda Z_k^2)] =& 1 + \int_1^\infty \mathbb{P}_{\nu^n}\Bab{e^{\lambda Z_k^2} > t}dt \\
    =& 1 + \int_1^\infty \inf_{\kappa > 0}\mathbb{P}_{\nu^n}\Bab{e^{\kappa\sqrt{\lambda}\ab|Z_k|} > e^{\kappa\ln^{1/2}(t)}}dt \\
    \le& 1 + 2\int_1^\infty \inf_{\kappa > 0}\exp\ab(\frac{\lambda \sigma^2\kappa^2}{2n} - \kappa\ln^{1/2}(t))dt \\
    =& 1 + 2\int_1^\infty \exp\ab(-\frac{n\ln(t)}{2\lambda\sigma^2}) dt \\
    =& 1 + 2\int_1^\infty t^{-\frac{n}{2\lambda\sigma^2}} dt \\
    =& 1 + \frac{2}{\frac{n}{2\lambda\sigma^2} - 1},
  \end{align}
  provided that $2\lambda\sigma^2 < n$. Therefore,
  \begin{align}
    \mathbb{E}_{\nu^n}\ab[\exp\ab(\lambda\max_kZ_k^2)] \le& \sum_k\mathbb{E}_{\nu^n}\ab[\exp\ab(\lambda Z_k^2)] \\
    \le& N_\delta \ab(1 + \frac{2}{\frac{n}{2\lambda\sigma^2} - 1}).
  \end{align}
  By Jensen's inequality, it follows that
  \begin{align}
    \mathbb{E}_{\nu^n}\ab[\lambda\max_kZ_k^2] \le \ln\ab(\mathbb{E}_{\nu^n}\ab[\exp\ab(\lambda\max_kZ_k^2)]) \le \ln(N_\delta) + \ln\ab(1 + \frac{2}{\frac{n}{2\lambda\sigma^2} - 1}).
  \end{align}
  Setting $\lambda = \frac{n}{4\sigma^2}$, we obtain
  \begin{align}
    \mathbb{E}_{\nu^n}\ab[\max_kZ_k^2] \le \frac{4\sigma^2}{n}\ln\ab(3N_\delta). \label{eq:bound-emp-error-4}
  \end{align}

  Substituting \zcref{eq:bound-emp-error-2}, \zcref{eq:bound-emp-error-3}, and \zcref{eq:bound-emp-error-4} into \zcref{eq:bound-emp-error-1}, we have
  \begin{align}
    \mathbb{E}_{\nu^n}\ab[\Vab{f_n - f^*}_n^2] \le \Vab{f - f^*}_{L^2(Q_X)}^2 + \sqrt{2\pi}\sigma\delta + 4\sqrt{\ab(\mathbb{E}_{\nu^n}\ab[\Vab{f_n - f^*}_n^2] + 4F\delta)\frac{\sigma^2}{n}\ln\ab(3N_\delta)}. \label{eq:bound-emp-error-5}
  \end{align}
  Applying \zcref{eq:quadratic-inequality} to \zcref{eq:bound-emp-error-5} yields
  \begin{multline}
    \mathbb{E}_{\nu^n}\ab[\Vab{f_n - f^*}_n^2] \le \Vab{f - f^*}_{L^2(\nu_X)}^2 + \sqrt{2\pi}\sigma\delta \\ + \frac{8\sigma^2\ln\ab(3N_\delta)}{n} + \ab(\frac{16\sigma^2\ln\ab(3N_\delta)}{n})^{1/2}\ab(\Vab{f - f^*}_{L^2(\nu_X)}^2+ \frac{4\sigma^2\ln\ab(3N_\delta)}{n} + \ab(\sqrt{2\pi}\sigma + 4F)\delta )^{1/2}.
  \end{multline}
  By the AM-GM inequality, for any $\zeta > 0$, we have
  \begin{align}
    \mathbb{E}_{\nu^n}\ab[\Vab{f_n - f^*}_n^2] \le (1+\zeta)\Vab{f - f^*}_{L^2(\nu_X)}^2 + \ab((1+\zeta)\sqrt{2\pi}\sigma + 4F\zeta)\delta + \ab(1 + \frac{\zeta}{2} + \frac{1}{2\zeta})\frac{8\sigma^2\ln\ab(3N_\delta)}{n}.
  \end{align}
  The arbitrariness of $\zeta$ gives the desired result.
\end{proof}

\subsection{Proof of \zcref{prop:composition-cv-bound-A}}\label{sec:proof-of-prop:composition-cv-bound-A}

\begin{proof}[Proof of \zcref{prop:composition-cv-bound-A}]
  Let $f = g_h \circ g_a \circ g_e \in \mathcal{F}$ and $f' = g_h \circ g'_a \circ g_e \in \mathcal{F}$. Then, we have
  \begin{align}
    \Vab{f - f'}_{L^\infty(\nu)} =& \Vab{g_h \circ g_a \circ g_e - g_h \circ g'_a \circ g_e}_{L^\infty(\nu)} \\
    =& \sup_{z \in g_e \circ \mathcal{X}} \ab|g_h(g_a(z)) - g_h(g'_a(z))| \\
    \le& C_\alpha\sup_{z \in \mathcal{Z}_i} \Vab{g_a(z) - g'_a(z)}^\alpha = C_\alpha\Vab{g_a - g'_a}_{L^\infty}^\alpha,
  \end{align}
  where $g_e \circ \mathcal{X} = \Bab{g_e(x): x \in \mathcal{X}} \subseteq \mathcal{Z}_i$. Let $g_1,...,g_N$ be a $\ab(\frac{\delta}{C_\alpha})^{1/\alpha}$-net for $\mathcal{G}$ with respect to $\Vab{\cdot}_{L^\infty}$. Then, for any $f = g_h \circ g'_a \circ g_e \in \mathcal{F}$, we have
  \begin{align}
    \min_{i \in \Bab{1,...,N}} \Vab{f - g_h \circ g_i \circ g_e}_{L^\infty(\nu)} \le C_\alpha\min_{i \in \Bab{1,...,N}} \Vab{g_a - g_i}_{L^\infty}^\alpha \le \delta.
  \end{align}
  Hence, $\Bab{g_h \circ g_i \circ g_e: i \in \Bab{1,...,N}}$ is a $\delta$-net for $\mathcal{F}$ with respect to $\Vab{\cdot}_{L^\infty(\nu)}$. Therefore, $N$ is greater than or equal to the covering number of $\mathcal{F}$ for the $\Vab{\cdot}_{L^\infty(\nu)}$-norm.
\end{proof}

\subsection{Proof of \zcref{prop:approximation-bound-A}}\label{sec:proof-of-prop:approximation-bound-A}

\begin{proof}[Proof of \zcref{prop:approximation-bound-A}]
  Let $f = g_h \circ g_a \circ g_e \in \mathcal{F}$ and $f^* = g_h \circ g^*_a \circ g_e$. Then, we have
  \begin{align}
    \Vab{f - f^*}_{L^2(Q_X)}^2 =& \int \ab((g_h \circ g_a \circ g_e)(x) - (g_h \circ g^*_a \circ g_e)(x))^2 Q_X(dx) \\
    \le& C_\alpha^2\int\ab(g_a(g_e(x)) - g^*_a(g_e(x)))^{2\alpha} Q_X(dx) \\
    \le& C_\alpha^2\int\Vab{g_a - g^*_a}_{L^\infty}^{2\alpha}Q_X(dx) = C_\alpha^2\Vab{g_a - g^*_a}_{L^\infty}^{2\alpha}.
  \end{align}
  Taking the square root of both sides, we get the desired claim.
\end{proof}

\subsection{Auxiliary Lemmas}
\subsubsection{Maximal Inequality for Finite Sum of Squares}
We present here a maximal inequality that is useful for controlling the supremum of a finite collection of random variables, which arises in the analysis of covering numbers and empirical processes. The following lemma provides an upper bound on the expected maximum of the squared values of a finite sequence of independent random variables, under suitable moment generating function conditions.
\begin{lemma}\label{lem:maximal-finite-sum-squares}
  Let $Z_1,...,Z_N$ be independent random variables on a probability space $(\mathcal{Z}, \mathfrak{Z}, \nu)$ such that $Z_i$ are sub-Gamma with variance proxy $\sigma^2$ and scaling factor $F > 0$, i.e., for all $\lambda \in (0, 1/F)$, $\mathbb{E}_\nu[e^{\lambda \ab|Z_k|}] \le \exp\ab(\frac{\lambda^2\sigma^2}{2(1-F\lambda)})$. Then,
  \begin{align}
    \mathbb{E}_{\nu}\ab[\max_{i \in \Bab{1,...,N}} Z_i^2] \le 16\inf_{\gamma > 0}\ab(\sigma^2\ln(N/\gamma) + \gamma\max\Bab{\sigma^2,2F^2\ln(N/\gamma)} + 2\gamma F^2).
  \end{align}
\end{lemma}
\begin{proof}[Proof of \zcref{lem:maximal-finite-sum-squares}]
  The expectation can be expressed in terms of probability as
  \begin{align}
    \mathbb{E}_{\nu}\ab[\max_{k \in \Bab{1,...,N}} Z_k^2] = \int_0^\infty \mathbb{P}_{\nu}\ab[\max_{k \in \Bab{1,...,N}} Z_k^2 > t] dt.
  \end{align}
  For any $z > 0$, we have
  \begin{align}
    \int_0^\infty \mathbb{P}_{\nu}\ab[\max_{k \in \Bab{1,...,N}} Z_k^2 > t] dt \le& z^2 + \int_{z^2}^\infty \mathbb{P}_{\nu}\ab[\max_{k \in \Bab{1,...,N}} \ab|Z_k| > \sqrt{t}] dt \\
    =& z^2 + \int_{z^2}^\infty \mathbb{P}_{\nu}\ab[\max_{k \in \Bab{1,...,N}} \exp\ab(\lambda \ab|Z_k|) > \exp\ab(\lambda\sqrt{t})] dt \\
    \le& z^2 + \sum_{i=1}^N \int_{z^2}^\infty \mathbb{P}_{\nu}\ab[\exp\ab(\lambda \ab|Z_k|) > \exp\ab(\lambda\sqrt{t})] dt \\
    \le& z^2 + \sum_{i=1}^N  \int_{z^2}^\infty \inf_{\lambda \in (0, 1/F)}\mathbb{E}_{\nu}\ab[\exp\ab(\lambda \ab|Z_k|)]e^{-\lambda\sqrt{t}} dt \\
    \le& z^2 + N\int_{z^2}^\infty \inf_{\lambda \in (0, 1/F)}\exp\ab(\frac{\lambda^2\sigma^2}{2(1-F\lambda)}-\lambda\sqrt{t}) dt \\
    \le& z^2 + N\int_{z^2}^\infty \exp\ab(-\frac{t}{2(\sigma^2 + F \sqrt{t})}) dt,
  \end{align}
  where the fourth inequality uses Markov's inequality, and the last inequality follows from the standard proof for the Bernstein inequality. Regarding the integral part, we have
  \begin{align}
    \int_{z^2}^\infty \exp\ab(-\frac{t}{2(\sigma^2 + F \sqrt{t})}) dt \le& \max\Bab{\int_{z^2}^\infty\exp\ab(-\frac{t}{4\sigma^2}) dt, \int_{z^2}^\infty \exp\ab(-\frac{\sqrt{t}}{4F}) dt} \\
    =& \max\Bab{4\sigma^2e^{-z^2/4\sigma^2},\ab(8F z + 32F^2)e^{-z/4F}}.
  \end{align}
  Setting $z = 4\max\Bab{\sigma\ln^{1/2}(N/\gamma),F\ln(N/\gamma)}$, we have
  \begin{align}
    & \mathbb{E}_{\nu}\ab[\max_{i \in \Bab{1,...,N}} Z_i^2] \\
    \le& 16\ln(N/\gamma)\max\Bab{\sigma^2,F^2\ln(N/\gamma)} + \gamma\max\Bab{4\sigma^2,32F(\sigma\ln^{1/2}(N/\gamma) + F),32F(F\ln(N/\gamma) + F)}.
  \end{align}
  If $\sigma^2 \le F^2\ln(N/\gamma)$, we have
  \begin{align}
    & \mathbb{E}_{\nu}\ab[\max_{i \in \Bab{1,...,N}} Z_i^2] \\
    \le& 16\sigma^2\ln(N/\gamma) + \gamma\max\Bab{4F^2\ln(N/\gamma),32F(F\ln(N/\gamma) + F),32F(F\ln(N/\gamma) + F)} \\
    =& 16\sigma^2\ln(N/\gamma) + 32\gamma F(F\ln(N/\gamma) + F) \\
    =& 16\ab(\sigma^2 + 2\gamma F^2)\ln(N/\gamma) + 32\gamma F^2.
  \end{align}
  If $\sigma^2 > F^2\ln(N/\gamma)$, we have
  \begin{align}
    & \mathbb{E}_{\nu}\ab[\max_{i \in \Bab{1,...,N}} Z_i^2] \\
    \le& 16\sigma^2\ln(N/\gamma) + \gamma\max\Bab{4\sigma^2,32(\sigma^2 + F^2),32(\sigma^2 + F^2)} \\
    =& 16\sigma^2\ln(N/\gamma) + 32\gamma (\sigma^2 + F^2).
  \end{align}
  Hence, we have
  \begin{align}
    \mathbb{E}_{\nu}\ab[\max_{i \in \Bab{1,...,N}} Z_i^2] \le 16\sigma^2\ln(N/\gamma) + 16\gamma\max\Bab{\sigma^2,2F^2\ln(N/\gamma)} + 32\gamma F^2,
  \end{align}
  which is equivalent to the desired result.
\end{proof}

\section{Detailed Example: Compositional Space}\label{sec:example-detailed}
In this section, we present a detailed application of \zcref{thm:error-bound} to the compositional spaces introduced in \zcref{sec:example}. Recall that the compositional space is defined as
\begin{align}
  \mathcal{F}_H \circ ... \circ \mathcal{F}_1 = \Bab{f_H \circ ... \circ f_1 :  f_i \in \mathcal{F}_i, i = 1, ..., H}.
\end{align}
Below, we provide rigorous definitions for the compositional sparse H\"older space~\citep{schmidt-hieberNonparametricRegressionUsing2020,kohlerRateConvergenceFully2021a} and the compositional anisotropic Besov space~\citep{suzukiDeepLearningAdaptive2021}.

\paragraph{Compositional Sparse H\"older Space.}
The compositional sparse H\"older space corresponds to the case where each $\mathcal{F}_i$ is a sparse H\"older space. For a multi-index $\alpha \in \mathbb{N}^d$, we denote $\partial^\alpha = \partial^{\alpha_1} ... \partial^{\alpha_d}$ and $\ab|\alpha| = \Vab{\alpha}_1$. For $\beta > 0$, the H\"older space is defined as
\begin{align}
  H^\beta([0,1]^d) = \Bab{f: [0,1]^d\to[0,1] : \Vab*{f}_{H^\beta([0,1]^d)} < \infty}.
\end{align}
where $\Vab*{f}_{H^\beta([0,1]^d)} = \max_{\alpha \in \mathbb{N}^d : \ab|\alpha| \le \beta}\Vab*{\partial^\alpha f}_{L^\infty([0,1]^d)} + \max_{\alpha \in \mathbb{N}^d : \ab|\alpha| = \lfloor\beta\rfloor}|\partial^\alpha f|_{H^{\beta-\lfloor\beta\rfloor}([0,1]^d)}$,
\begin{align}
  |f|_{H^\beta([0,1]^d)} = \sup_{x,y \in (0,1)^d : x \ne y} \frac{\ab|f(x) - f(y)|}{\Vab{x - y}^\beta}.
\end{align}
The compositional sparse H\"older space is given by \zcref{eq:compositional-space} with
\begin{align}
  \mathcal{F}_i = \mathcal{H}_i = \Bab{f = (f_j)_j:[0,1]^{d_j}\to[0,1]^{d_{j+1}} : f_j \in H^{\beta_j}([0,1]^{t_j}), \Vab{f_j}_{H^{\beta_j}([0,1]^{t_j})} \le 1},
\end{align}
for some $\beta_j > 0$ and $1 \le t_j \le d_j$.

\paragraph{Compositional Anisotropic Besov Space.}
The compositional anisotropic Besov space arises when each $\mathcal{F}_i$ is an anisotropic Besov space.
For a function $f:[0,1]^d\to[0,1]$, the $r$th finite difference is defined recursively as
\begin{align}
  \Delta^r_h(f)(x) = \Delta^{r-1}_{h}(f)(x + h) - \Delta^{r-1}_{h}(f)(x), \Delta^0_h(x) = f(x),
\end{align}
where $x \in [0,1]^d$ and $x + rh \in [0,1]^d$, and otherwise $\Delta^r_h(f)(x) = 0$.
For a function $f \in L^p([0,1]^d)$ and $p \in (0, \infty]$, the $r$th modulus of smoothness of $f$ is defined as
\begin{align}
  w_{r,p}(f, t) = \sup_{h \in \mathbb{R}^d : |h_i| \le t_i} \Vab{\Delta^r_h(f)(x)}_{L^p([0,1]^d)},
\end{align}
where $t = (t_1, ..., t_d)^\top \in \mathbb{R}^d, t_i > 0$.
For $\beta = (\beta_1,...,\beta_d)^\top \in \mathbb{R}^d$, where $\beta_i > 0$, the anisotropic Besov space is defined by
\begin{align}
  B^\beta_{p,q}([0,1]^d) = \Bab*{f \in L^p : \Vab*{f}_{B^\beta_{p,q}([0,1]^d)} < \infty},
\end{align}
where $\Vab*{f}_{B^\beta_{p,q}([0,1]^d)} = \Vab*{f}_{L^p} + |f|_{B^\beta_{p,q}([0,1]^d)}$, and
\begin{align}
  |f|_{B^\beta_{p,q}([0,1]^d)} =
  \begin{dcases}
    \ab(\sum_{k=0}^\infty\ab(2^kw_{r,p}\ab(f, \ab(2^{-k/\beta_1}, ..., 2^{-k/\beta_d})))^q)^{1/q} & \text{if } q < \infty, \\
    \sup_{k \ge 0} 2^kw_{r,p}\ab(f, \ab(2^{-k/\beta_1}, ..., 2^{-k/\beta_d})) & \text{if } q = \infty.
  \end{dcases}
\end{align}
The compositional anisotropic Besov space is given by \zcref{eq:compositional-space} with
\begin{align}
  \mathcal{F}_i = \mathcal{B}_i = \Bab{f :[0,1]^{d_i} \to [0,1]^{d_{i+1}}, f_j \in B^{\beta_i}_{p,q}([0,1]^{d_i}), \Vab{f}_{B^{\beta_i}_{p,q}([0,1]^{d_i})} \le 1},
\end{align}
where $\beta_i > 0$.

\paragraph{Sparse ReLU Network.}
Sparse ReLU networks are used in~\cite{schmidt-hieberNonparametricRegressionUsing2020,kohlerRateConvergenceFully2021a,suzukiDeepLearningAdaptive2021} to effectively approximate the compositional spaces.
A ReLU network is a deep neural network with ReLU activation functions $\eta(x) = \max(0, x)$.
For a vector $x \in \mathbb{R}^d$, $\eta(x)$ is applied element-wise.
The class of ReLU networks with height $N_h$, width $N_w$, sparsity constraint $S$, and norm constraint $B$ is defined as
\begin{multline}
  \mathcal{F}_{\mathrm{ReLU}}(N_h, N_w, S, B) = \Bab[Bigg]{ (W^{(N_h)}\eta(\cdot)+b^{(L)})\circ ... \circ (W^{(1)}\eta(\cdot)+b^{(1)}) : \\ W^{(L)} \in \mathbb{R}^{1\times N_w}, b^{(L)} \in \mathbb{R}, W^{(1)} \in \mathbb{R}^{N_w\times d}, W^{(\ell)} \in \mathbb{R}^{N_w\times N_w}, b^{(\ell)} \in \mathbb{R}^{N_w}, \\ \sum_{\ell=1}^{N_h} \ab(\Vab{W^{(\ell)}}_\infty + \Vab{b^{(\ell)}}_\infty) \le S, \max_{\ell=1,...,N_h} \Vab{W^{(\ell)}}_\infty \lor \Vab{b^{(\ell)}}_\infty \le B}.
\end{multline}
The approximation error of $\mathcal{F}_H\circ...\circ\mathcal{F}_1$ with $H^{\beta_i}([0,1]^{t_i})$ and $B^{\beta_i}_{p,q}([0,1]^{d_i})$ by $\mathcal{F}_{\mathrm{ReLU}}(N_h, N_w, S, B)$, as well as the covering number of $\mathcal{F}_{\mathrm{ReLU}}(N_h, N_w, S, B)$, are established in the following theorems.
\begin{theorem}[\cite{schmidt-hieberNonparametricRegressionUsing2020}]\label{thm:complexity-of-comp-holder-space-by-relu-1}
  Define $\alpha_i = \prod_{\ell=i+1}^H\min\Bab{1, \beta_\ell}$. For each $J \in \mathbb{N}$, there exist appropriate choices of $N_h, N_w, S, B$ such that
  \begin{align}
    N(\delta, \mathcal{F}_{\mathrm{ReLU}}(N_h, N_w, S, B), \Vab{\cdot}_{L^\infty([0,1]^{d_1})}) \le CJ\mathrm{polylog}(J, 1/\delta),
  \end{align}
  and for any $f^* \in \mathcal{H}_H\circ...\circ\mathcal{H}_1$,
  \begin{align}
    \inf_{f \in \mathcal{F}_{\mathrm{ReLU}}(N_h, N_w, S, B)} \Vab{f - f^*}_{L^\infty([0,1]^{d_1})} \le C'J^{-\min_{i=1,...,H}\frac{\alpha_i\beta_i}{t_i}},
  \end{align}
  for some constants $C, C' > 0$.
\end{theorem}
\begin{theorem}[\cite{suzukiDeepLearningAdaptive2021}]\label{thm:complexity-of-comp-anis-besov-space-by-relu}
  Define $\tilde{\beta}^{(\ell)} = (\sum_{j=1}^d 1/\beta^{(\ell)}_j)^{-1}$ and $\alpha_i = \prod_{k=\ell+1}^H(\min\Bab*{1, (\beta^{(\ell)}_{\min} - 1/p)})$, where $\beta^{(\ell)}_{\min} = \min_{j=1,...,d} \beta^{(\ell)}_j$.
  Assume that $\tilde{\beta}^{(\ell)} > 1/p$ for all $\ell$.
  For each $J \in \mathbb{N}$, there exist appropriate choices of $N_h, N_w, S, B$ such that
  \begin{align}
    N(\delta, \mathcal{F}_{\mathrm{ReLU}}(N_h, N_w, S, B), \Vab{\cdot}_{L^\infty(\lambda)}) \le CJ\ln(J)\ab(\ln^2(J) + \ln(1/\delta)),
  \end{align}
  and for any $f^* \in \mathcal{B}_H\circ...\circ\mathcal{B}_1$,
  \begin{align}
    \inf_{f \in \mathcal{F}_{\mathrm{ReLU}}(N_h, N_w, S, B)} \Vab{f - f^*}_{L^\infty(\lambda)} \le C'J^{-\min_{\ell \in \Bab{1,...,H}} \alpha_i\tilde{\beta}^{(\ell)}},
  \end{align}
  for some constants $C, C' > 0$.
\end{theorem}

\paragraph{Error without a pre-trained model.}
By combining \zcref{thm:complexity-of-comp-holder-space-by-relu-1} or \zcref{thm:complexity-of-comp-anis-besov-space-by-relu} with \zcref{thm:error-bound}, we obtain the following error rate:
\begin{align}
  \mathbb{E}_{Q^n}[E_Q(f_n)] \le C \max_{i=1,...,H}n^{-\gamma_i}\mathrm{polylog}(n), \label{eq:error-bound-without-pre-trained-model-A}
\end{align}
where $\gamma_i = \frac{2\alpha_i\beta_i}{2\alpha_i\beta_i + t_i}$ for the compositional sparse H\"older space and $\gamma_i = \frac{\alpha_i\tilde{\beta}^{(\ell)}}{2\alpha_i\tilde{\beta}^{(\ell)} + 1}$ for the compositional anisotropic Besov space.

\paragraph{Error with a pre-trained model (caulking).}
Now, consider the scenario where a pre-trained model $f_{\mathrm{pre}}$ is available, and $f^*$ is $(1/n, \mathcal{F}_{i_h}\circ\cdots\circ\mathcal{F}_{i_e})$-caulkable by $f_{\mathrm{pre}} = (g_{h,m}, g_{e,m})$ for some $1 < i_e < i_h < H$, with $\mathcal{F}_H\circ...\circ\mathcal{F}_1$ being either the compositional sparse H\"older space or the compositional anisotropic Besov space. The complexity of the class $\mathcal{F}_{i_h}\circ...\circ\mathcal{F}_{i_e}$ is also established by \zcref{thm:complexity-of-comp-holder-space-by-relu-1} or \zcref{thm:complexity-of-comp-anis-besov-space-by-relu}. Consequently, we have
\begin{align}
  \mathbb{E}_{Q^n}[E_Q(f_n)] \le C \max_{i=i_e,...,i_h}n^{-\gamma_i}\mathrm{polylog}(n), \label{eq:error-bound-with-pre-trained-model-A}
\end{align}
Comparing \zcref{eq:error-bound-without-pre-trained-model-A} and \zcref{eq:error-bound-with-pre-trained-model-A}, we observe that the range of the index $i$ is restricted to a narrower subset by leveraging a pre-trained model, demonstrating the improved error rate achievable through empirical caulking with a pre-trained model.

\section{Details of Experiments}\label{app:sec:experimental_details}

\subsection{Finetuning of CNNs}

We use ResNet-50 (26M parameters, \cite{he2016deep}) and Wide ResNet-50-2 (68M parameters, \cite{zagoruyko2016wide}), pre-trained on ImageNet.
Their BatchNorm layers are replaced with GroupNorm~\citep{wu2018group}.
These CNNs with GroupNorm are trained on other domains than the target domain (clipart) of Office-Home dataset~\citep{venkateswaraDeepHashingNetwork2017}  and then finetune on the target domain.
Each domain is split into 80\% for training and 20\% for validation.

When fine-tuning, adapters are inserted in between the image-feature extractor and the last classifier.
Adapters are multi-layer perceptron with the ReLU activation.
We adopt SGD with a learning rate of $10^{-3}$, a momentum of $0.9$, a weight decay of $10^{-4}$ for pre-training on the other domain and AdamW with a learning rate of $10^{-5}$ and a weight decay of $10^{-2}$ for fine-tuning of the adapters.
At each stage, the models were updated for 50 epochs.
A single NVIDIA's H100 GPU in our internal cluster was used to run each trial.
The best results on the validation split over five different random seeds were reported.

The attached script is to reproduce the experiments. To run it, \texttt{uv}\footnote{See \url{https://docs.astral.sh/uv/} for the installation instruction.} is required.

\subsection{Integrating Vision Capabilities to LLMs}

We use \texttt{nanovlm v0.2}\footnote{\url{https://github.com/huggingface/nanoVLM/tree/v0.2}} for this experiment and follow its settings other than our problem-specific settings described below.
A Llama-3-style Transformer model, \texttt{SmolLM2-135M}~\citep{allal2024smollm}, is adopted as the base language model, and the vision encoders of SigLIP2~\cite{tschannen2025siglip2}, \texttt{siglip2-base-patch16-224} and \texttt{siglip2-base-patch16-224}, are used as the visual feature extractors.
The extracted visual features are projected to the language embedding space after applying pixel shuffle and adapters.

The adapters are optimized with AdamW with a learning rate of $10^{-2}$ and a weight decay of $10^{-2}$, and the language model is also optimized with AdamW with a learning rate of $10^{-5}$ and a weight decay of $10^{-2}$ for 64k iterations.
The weights of visual feature extractors are fixed.
To run each trial, we used eight H100 GPUs in our internal cluster.
The best test accuracy on MMStar was reported.

\end{document}